\documentclass[10pt,journal,compsoc]{IEEEtran}
\usepackage{xcolor,soul,framed} 
\colorlet{shadecolor}{yellow}
\usepackage[pdftex]{graphicx}
\graphicspath{{../pdf/}{../jpeg/}}
\DeclareGraphicsExtensions{.pdf,.jpeg,.png}
\usepackage[cmex10]{amsmath}
\usepackage{array}
\usepackage{mdwmath}
\usepackage{mdwtab}
\usepackage{eqparbox}
\usepackage{url}
\usepackage[english]{babel}
\usepackage[utf8]{inputenc}
\usepackage{fancyhdr}
\usepackage{ragged2e}
\usepackage[T1]{fontenc}
\usepackage{ amssymb }
\usepackage{mathtools}
\usepackage{ textcomp }
\usepackage{ textcomp }
\usepackage{ dsfont }
\usepackage{ upgreek }
\usepackage{amsmath}
\usepackage{graphicx}
\usepackage{float}
\usepackage{ tipa }
\usepackage{xcolor}
\usepackage[super]{nth}
\usepackage{caption}
\usepackage{subcaption}
\mathtoolsset{showonlyrefs}
\usepackage{algorithm,algorithmic}

\newcommand{\hatv}[1]{\overset{\wedge}{\mathstrut#1}}
\newcommand{\Whatv}{\hatv{W}}
\newcommand{\Hhatv}{\hatv{H}}
\newcommand\norm[1]{\left\lVert#1\right\rVert}
\newcommand*{\argmin}{\operatornamewithlimits{argmin}\limits}

\newcommand\numberthis{\addtocounter{equation}{1}\tag{\theequation}}

\def\vect{\mathbf}
\def\matr{\mathbf}

\makeatletter

\newcommand*{\rom}[1]{\expandafter\@slowromancap\romannumeral #1@}

\DeclareMathOperator{\tr}{tr}
\DeclarePairedDelimiter\set\{\}

\usepackage{amsthm}
\newtheorem{definition}{\textbf{Definition}}
\newtheorem{theorem}{Theorem}[section]

\ifCLASSOPTIONcompsoc
  \usepackage[nocompress]{cite}
\else
  \usepackage{cite}
\fi
\ifCLASSINFOpdf
\else
\fi
\begin{document}

%


\title{Privacy-preserving Non-negative Matrix Factorization with Outliers}
%
%
%
%

\author{Swapnil~Saha,~and~Hafiz~Imtiaz~\IEEEmembership{}
\IEEEcompsocitemizethanks{
\IEEEcompsocthanksitem Swapnil Saha is a graduate student with the Department
of Electrical and Electronic Engineering, Bangladesh University of Engineering and Technology, Dhaka, Bangladesh.\protect\\
E-mail: 1606095@eee.buet.ac.bd
\IEEEcompsocthanksitem Hafiz Imtiaz is an Associate Professor with the Department
of Electrical and Electronic Engineering, Bangladesh University of Engineering and Technology, Dhaka, Bangladesh.\protect\\
E-mail: hafizimtiaz@eee.buet.ac.bd.}
\thanks{Manuscript received XXXX XX, 2022; revised XXXX XX, XXXX.}}

%
%

\markboth{IEEE Transactions on Knowledge and Data Engineering,~Vol.~XX, No.~XX, XXXX~XXXX}%
{Saha and Imtiaz: Privacy-preserving Non-negative Matrix Factorization with Outliers}
%



\IEEEtitleabstractindextext{%
\begin{abstract}

Non-negative matrix factorization is a popular unsupervised machine learning algorithm for extracting meaningful features from data which are inherently non-negative. However, such data sets may often contain privacy-sensitive user data, and therefore, we may need to take necessary steps to ensure the privacy of the users while analyzing the data. In this work, we focus on developing a Non-negative matrix factorization algorithm in the privacy-preserving framework. More specifically, we propose a novel privacy-preserving algorithm for non-negative matrix factorisation capable of operating on private data, while achieving results comparable to those of the non-private algorithm. We design the framework such that one has the control to select the degree of privacy grantee based on the utility gap. We show our proposed framework’s performance in six real data sets. The experimental results show that our proposed method can achieve very close performance with the non-private algorithm under some parameter regime, while ensuring strict privacy.

\end{abstract}

}

\maketitle

\IEEEdisplaynontitleabstractindextext

%
\IEEEpeerreviewmaketitle

\IEEEraisesectionheading{\section{Introduction}\label{sec:introduction}}

%
%
%
%
\IEEEPARstart{N}{on-negative} matrix factorisation is an unsupervised machine learning technique for discovering the part-based representation of intrinsically non-negative data \cite{hoyer2004non}. For a $D \times N$ data matrix $\matr{V}$, where $N$ is the number of the data samples, and $D$ is the data dimension, the entries satisfy $v_{ij} \geqslant 0\ \forall i\in \{1, 2, \ldots, D\}$ and $\forall j\in \{1, 2, \ldots, N\}$. The NMF objective is to decompose the data matrix as the following:
\begin{equation}\label{VWH}
    \matr{V} \approx \matr{WH},
\end{equation}
where $\matr{W} \in \mathcal{R}^{D\times K}$ is the basis matrix, $\matr{H} \in \mathcal{R}^{K\times N}$ is the coefficient matrix, and $K$ is the reduced latent dimension. Here, each entry of $\matr{W}$ satisfies $w_{ik} \geqslant 0$, and each entry of $\matr{H}$ satisfies $h_{kj}\geqslant 0,\forall k\in\{1, 2, \ldots, K\}$. In short, NMF performs dimension reduction by mapping the ambient data dimension $D$ into reduced latent dimension $K$ for $N$ data samples. The $j$-th column of $\matr{V}$ can be written as:
\begin{equation}
    \vect{v}_j \approx \matr{W} \vect{h}_j,
\end{equation}
where $\vect{h}_j$ is the $j$-th column of $\matr{H}$. Essentially, the $j$-th column of $\matr{V}$ is represented as a linear combination of all columns of $\matr{W}$ with the coefficients being the corresponding entries from $\vect{h}_j$. Therefore, the dictionary matrix $\matr{W}$ can be interpreted as storing the ``parts'' of the data matrix $\matr{V}$. This part-based  decomposition, and considering only the non-negative values, make NMF popular in many practical applications including, but not limited to, dimension reduction, topic modeling in text mining, representation learning, extracting local facial features from human faces, unsupervised image segmentation, speech denoising, and community detection in social networks.

\noindent\textbf{NMF and Privacy. }\textcolor{blue}{In the modern era of big data, many services are customized for the users to provide better suggestions and experiences. Such customization is often done via some machine learning algorithm, which are trained or fine-tuned on users’ sensitive data. As shown in both theoretical and applied works, the users may be rightfully concerned regarding their privacy being compromised by the machine learning algorithm’s outputs. For example, the seminal work of Homer et al.~\cite{homer2008resolving} showed that the presence of an individual in a genome dataset can be identified from simple summary statistics about the dataset. In the machine learning setting, typically the model parameters (as a matrix $\matr{W}$ or vector $\vect{w}$) are learned by training on the users’ data. To that end, ``Membership Inference Attacks” \cite{fredrikson2015model} are discussed in detail by Hu et al.~\cite{hu2022membership}, and Shokri et al.~\cite{shokri2017membership}. They showed that given the learned parameters $\matr{W}$, an adversary can identify users in the training set. Basically, the model's trained weights can be used to extract sensitive information. The weights' tendency to memorize the training examples is used to regenerate the training example. Several other works~\cite{narayanan2006break,le2013differentially,sweeney2015only} also showed how personal data leakage occurred from modern machine learning (ML) and signal processing tasks. Additionally, it has been shown that simple anonymization of data does not provide any privacy in the presence of auxiliary information available from other/public sources~\cite{sweeney2015only}. One of the examples is the Netflix prize, where an anonymized data set was released in 2007. This data set contained anonymized movie ratings from Netflix subscribers. Nevertheless, researchers found a way to crack the privacy of the dataset and successfully recovered $99\%$ of the removed personal data \cite{narayanan2006break} by using the publicly available IMDb datasets. In addition, privacy can leak through gradient sharing \cite{zhu2019deep} as well. In the distributed machine learning system, several nodes exchange gradient values. The authors in~\cite{zhu2019deep} showed how one can extract private training information from gradient sharing. These discoveries led to widespread concern over using private data for public machine learning algorithms.} Evidently, personal information leakage is the main hindrance to collecting and analyzing sensor data for training machine learning and signal processing algorithms. It is, therefore, necessary to develop a framework where one can share private data without disclosing their participation or identity. Differential privacy (DP) is a rigorous mathematical framework that can protect against information leakage \cite{dwork2014algorithmic}. The definition of differential privacy is motivated by the cryptographic work, which has gained significant attention in machine learning and data mining communities \cite{dwork2006calibrating}. Though the target may be the same: give privacy to data, how DP guarantees privacy is much different from the cryptography~\cite{vaidya2006privacy} and information theory~\cite{zhou2013achieving}. In a differentially private mechanism, we can learn useful information from the population while ensuring the privacy of the population members. It is done statistically by introducing randomness in the algorithm. This randomness provides confusion to attack the sensitive data of the participants. However, one may need to compromise the algorithm’s utility to ensure privacy. That is, one needs to quantitatively choose the optimum privacy budget considering the required privacy-utility trade-off.

\noindent\textbf{Related Works. }The non-negative matrix factorization is attained in the literature by minimizing the following objective function: $\min_{\matr{W} \in \mathcal{C}, h_{kj}\geqslant 0, \forall k,j} \norm{\matr{V} - \matr{WH}}_F^2$,
where $\mathcal{C} \subseteq \mathcal{R}^{D\times K}$ is the constraint set for $\matr{W}$. Several algorithms have been proposed to obtain the optimal point of this objective function, such as the multiplicate updates~\cite{NIPS2000_f9d11525}, alternating direction method of multipliers~\cite{xu2012alternating}, block principal pivoting~\cite{kim2008toward}, active set method~\cite{kim2008nonnegative}, and projected gradient descent~\cite{lin2007projected}. Most of these algorithms are based on alternatively updating $\matr{W}$ and $\matr{H}$. This makes the optimization problem divided into two sub-problems: each of which can be optimized using the standard optimization techniques, such as the projected gradient or the interior point method. A detailed survey of these optimization techniques can be found in~\cite{wang2012nonnegative, kim2014algorithms}. Our work is based on the robust NMF algorithm using projected gradient descent~\cite{zhao2016online}. This modified robust algorithm improves two extreme scenarios: (i) when the data matrix $\matr{V}$ has a large number of data samples, and (ii) the existence of outliers in data samples. The second scenario is common in many practical cases, such as the salt and pepper noise in image data, and impulse noise in time series data. If these outliers and noises are not handled properly during matrix decomposition, the basis (or dictionary) matrix $\matr{W}$ may not be well optimized and fail to learn the part-based representation. We discuss the implementation of the algorithm in  detail in the Section \ref{nmf_problem_formulation}.
 
Extensive works and surveys exist in the literature on differential privacy. In particular, Dwork and Smith's survey \cite{dwork2010differential} contains the earlier theoretical work. We refer the reader to \cite{abadi2016deep,chaudhuri2008privacy,chaudhuri2011differentially,song2013stochastic,ji2014differential,li2018differentially,bassily2014private,ligett2017accuracy,wang2017differentially} for the most relevant works in differentially private machine learning, deep learning, optimization problems, gradient descent, and empirical risk minimization. Adding randomness in the gradient calculation is one of the most common approaches for implementing differential privacy~\cite{song2013stochastic,bassily2014private}. Other common approaches are employing the output~\cite{chaudhuri2011differentially} and objective perturbations~\cite{nozari2016differentially}, the exponential mechanism~\cite{mcsherry2007mechanism}, the Laplace mechanism~\cite{dwork2006calibrating} and Smooth sensitivity~\cite{nissim2007smooth}. \textcolor{blue}{Last but not the least, the work of Alsulaimawi~\cite{alsulaimawi2020non} has introduced a privacy filter with federated learning for NMF factorization. Fu et al.~\cite{fu2019cloud} implemented  privacy-preserving NMF for dimension reduction using Paillier Cryptosystem, Nikolaenko et al.~\cite{nikolaenko2013privacy} did privacy preserving matrix factorization for recommendation systems using partially homomorphic encryption with Yao’s garbled circuits. Privacy Preserving Data Mining (PPDM) was implemented in \cite{afrin2019privacy} using combined NMF and Singular Value Decomposition (SVD) methods. The authors in~\cite{ran2022differentially} proposed differentially private NMF for the recommender system, and we showed the comparison analyses in Section \ref{Comparison}. However, to the best of our knowledge, no work has introduced differential privacy in the universal NMF decomposition and calculated privacy composition for multi-stage implementation to account for the best privacy budget.}

\noindent\textbf{Our Contributions. }\textcolor{blue}{In this work, we intend to perform NMF decomposition on inherently non-negative and privacy-sensitive data. As the data matrix $\matr{V}$ contains user-specific information, an adversary can extract sensitive information regarding the users from the estimated dictionary matrix $\matr{W}$. However, the estimated $\matr{W}$ should encompass the fundamental basis of the population. Note that, the data may have some outliers, which may cause one to capture unrepresentative dictionary matrix if the outliers are not handled properly. To that end, we propose a privacy-preserving non-negative matrix factorization algorithm considering outliers. We propose to compute the dictionary matrix $\matr{W}$ satisfying a mathematically rigorous privacy guarantee, differential privacy, such that the computed $\matr{W}$ reflects very little about any particular user's data, is relatively unaffected by the presence of outliers, and closely approximates the true dictionary matrix. Our major contribution is summarized below: 
\begin{itemize}
    \item We develop a novel privacy-preserving algorithm for non-negative matrix factorization capable of operating on privacy-sensitive data, while closely approximating the results of the non-private algorithm.
    \item We consider the effect of outliers by specifically modeling them, such that the presence of outliers have very little effect on the estimated dictionary matrix $\matr{W}$.
    \item We analyze our algorithm with R\'enyi Differential Privacy \cite{mironov2017renyi} to obtain a much better accounting of the overall privacy loss, compared to the conventional strong composition rule~\cite{dwork2014algorithmic}. 
    \item We performed extensive experimentation on real datasets to show the effectiveness of our proposed algorithm. We compare the results with those of the non-private algorithm and observe that our proposed algorithm can offer close approximation to the non-private results for some parameter choices.
    \item We present the result plots in a way that the user can choose between the overall privacy budget and the required ``closeness'' to non-private results (utility-gap).
\end{itemize}
}

\section{Problem formulation}
\subsection{Notations}

\begin{table}[t]
    \centering
    \begin{tabular}{c c}
         \hline
         Notation & Meaning \\ 
         \hline
         $\matr{V}$ & Data Matrix \\
         $\matr{W}$ & Dictionary / Basis Matrix \\
         $\matr{W}_{\textrm{private}}$ & Differentially private $\matr{W}$  \\
         $\matr{H}$ &  Coefficient Matrix\\
         $\matr{R}$ &  Outlier Matrix \\
         $D$ & Ambient dimension of Data matrix \\
         $K$ & Latent dimension   \\
         $N$ & Number of users in Data matrix \\
         $\vect{v}_n$ & $n$-th user vector in Data matrix \\
         $(\epsilon,\delta)$ & Privacy parameters \\
         $\Updelta$ & Sensitivity \\
         $\triangledown f(W)$ & Gradient function to update Dictionary matrix \\
         $\overline{\triangledown f(W)}$ & Noisy gradient to learn $\matr{W}_{\textrm{private}}$ \\
         $\matr{A},\matr{B}$ & Statistics matrices of $\triangledown f(W)$ \\
         $\overline{\matr{A}},\overline{\matr{B}}$ & Noise perturbed  statistics matrices of $\matr{A},\matr{B}$\\

         \hline
               
    \end{tabular}
    \caption{Notations}
    \label{tab:notation}
\end{table}

For clarity and readability, we denote vector, matrix and scalar with different notation. Bold lower case letter $(\textbf{v})$, bold capital letter $(\textbf{V})$ and unbolded letter $(M)$ are used respectively for vector, matrix and scalar. To indicate the iteration instant, we use subscript $t$. For example, $\matr{W}_t$ denotes the dictionary matrix after $t$ iterations. The superscript $'+'$ indicates a single update. The $n$-th column of matrix $\textbf{V}$ is denoted as $\textbf{v}_n$. We denote the indices with lowercase unbolded letter. For example, $v_{ij}$ indicates the entry of the $i$-th row and $j$-th column of the matrix $\matr{V}$. Inequality $\vect{x} \geqslant 0$ or $\matr{X} \geqslant 0$ apply entry-wise. For element wise matrix multiplication, we use the notation $\odot$. We denote the $\mathcal{L}_2$ norm (Euclidean norm) with $\norm{.}_2$, the $\mathcal{L}_{1,1}$ norm with $\norm{.}_{1,1}$, and the Frobenius norm with $\norm{.}_F$. $\mathds{R} \text{ , and } \mathds{R}_{+}$ denote the set of real numbers, and set of positive real numbers, respectably. $\mathcal{P}_+$ denotes the Euclidean projector, which projects value onto the non-negative orthant. Lastly, the probability distribution function of zero mean unit variance Gaussian random variable is given as follows: $f(x)=\frac{1}{\sqrt{2\pi}}\exp{\frac{-x^2}{2}}$. We summarized the frequently used notations at Tab. \ref{tab:notation}.
 
\subsection{Definitions and Preliminaries} \label{definations}
In differential privacy, we define $\mathcal{D}$ as the domain of the databases consisting of $N$ records. We also define neighboring data sets $D,\ D^\prime$, which differ by only one record.

\begin{definition}
  (($\epsilon,\delta$)-DP \cite{dwork2006calibrating}) An algorithm $f$ : $\mathcal{D} \mapsto \mathcal{T}$ provides ($\epsilon,\delta$)- differential privacy (($\epsilon,\delta$)-DP) if $P(f(D)\in \mathcal{S}) \leqslant \delta + e^\epsilon P(f(D^\prime)\in \mathcal{S} )$ for all measurable $\mathcal{S} \subseteq \mathcal{T} $ and for all neighbouring data sets $D,D^\prime \in \mathcal{D}$.
\end{definition}
Here, $\epsilon,\delta \geqslant 0$ are the  privacy parameters and determine how the algorithm will perform in providing utility and preserving privacy. The $\epsilon$ indicates how much the algorithm’s output deviates in probability, when we replace one single person’s data with another. The parameter $\delta$ indicates the probability that the privacy mechanism fails to give the guarantee of $\epsilon$. Intuitively, higher privacy demand makes poor utility. A lower value of  $\epsilon$ and $\delta$ guarantee more privacy but lower utility. There are several mechanisms to implement differential privacy: Gaussian \cite{dwork2006calibrating}, Laplace mechanism \cite{dwork2014algorithmic}, random sampling, and exponential mechanism \cite{mcsherry2007mechanism} are well-known. Among the additive noise mechanisms, the noise’s standard deviation is scaled by the privacy budget and the sensitivity of the function.

\begin{definition}
($\mathcal{L}_2$ sensitivity \cite{dwork2006calibrating}) The $\mathcal{L}_2$- sensitivity of vector valued function $f(D)$ is $\Updelta := \max_{D,D^\prime} \norm{f(D)-f(D^\prime)}_{2}$, where $D$ and $D^\prime$  are the neighbouring dataset.
\end{definition} 
The $\mathcal{L}_2$ sensitivity of a function gives the upper bound of how much randomness we need to perturb to the function’s output, if we want to keep the guarantee of differential privacy. It captures the maximum output change by a single user in the worst-case scenario.
\begin{definition}
(Gaussian Mechanism \cite{dwork2014algorithmic}) Let $f : \mathcal{D} \mapsto \mathcal{R}^d $ be an arbitrary function with $\mathcal{L}_{2}$  sensitivity $\Updelta$. The Gaussian mechanism with parameter $\tau$ adds noise from $\mathcal{N}(0,\tau^2)$ to each of the $d$ entries of the output and satisfies $(\epsilon,\delta)$ differential privacy for $\epsilon \in (0,1)$ and $\delta \in (0,1)$, if $\tau \geq \frac{\Updelta}{\epsilon} \sqrt{2\log\frac{1.25}{\delta}}$.
\end{definition} 
Here, $(\epsilon,\delta)$- differential privacy is guaranteed by adding noise drawn form $\mathcal{N}(0,\tau^2)$ distribution. Note that, there is infinite combinations of $(\epsilon,\delta)$ for a given $\tau^2$.
\begin{definition}
(R\'enyi Differential Privacy \cite{mironov2017renyi}) A randomized algorithm $f$ : $\mathcal{D} \mapsto \mathcal{T}$ is $(\alpha,\epsilon_{r})$-R\'enyi differentially private if, for any adjacent $D,D^\prime \in \mathcal{D}$, the following holds: $D_{\alpha}(\mathcal{A}(D)\ || \mathcal{A}(D^\prime)) \leq \epsilon_{r}$. Here, $D_{\alpha}(P(x)||Q(x))=\frac{1}{\alpha-1}\log\mathds{E}_{x \sim Q} \Big(\frac{P(x)}{Q(x)}\Big)^{\alpha}$ and $P(x)$ and $Q(x)$ are probability density functions defined on $\mathcal{T}$.
\end{definition} 
We use R\'enyi Differential Privacy for calculating the total privacy budget spent in multi-stage differentially private mechanisms. RDP provides a much simpler rule for calculating overall privacy risk $\epsilon$ that is shown to be tight \cite{mironov2017renyi}.

\newtheorem{prop}{Proposition}
\begin{prop}\label{rdp_dp}
(From RDP to DP \cite{mironov2017renyi}). If $f$ is an $(\alpha,\epsilon_r)$-RDP mechanism, it also satisfies $(\epsilon_r+\frac{\log 1/\delta}{\alpha-1},\delta)$-differential privacy for any $0<\delta<1$.
\end{prop}

\begin{prop}\label{comp_rdp}
(Composition of RDP \cite{mironov2017renyi}). Let $f_1:\mathcal{D}\rightarrow \mathcal{R}_1$ be $(\alpha,\epsilon_1)$ and $f_2:{\mathcal{R}_1} \times \mathcal{D} \rightarrow \mathcal{R}_2$ be $(\alpha,\epsilon_2)$-RDP,then the mechanism defined as $(X_1,X_2)$, where $X_1 \sim f_1(\mathcal{D})$ and $X_2 \sim f_2(X_1,\mathcal{D})$ satisfies $(\alpha,\epsilon_1+\epsilon_2)$-RDP.
\end{prop}

\begin{prop}\label{rdp_gaussian}
(RDP and Gaussian Mechanism\cite{mironov2017renyi}). If $f$ has $\mathcal{L}_2$ sensitivity 1, then the Gaussian mechanism $\mathcal{G}_{\sigma}f(\mathcal{D})=f(\mathcal{D})+\mathcal{E}$ where $\mathcal{E} \sim \mathcal{N}(0,\sigma^2)$ satisfies $(\alpha,\frac{\alpha}{2\sigma^2})$-RDP. Additionally, a composition of $K$ such Gaussian mechanisms satisfies $(\alpha,\frac{\alpha K}{2 \sigma^2})$- RDP.
\end{prop}

\subsection{NMF Problem Formulation} \label{nmf_problem_formulation}

\begin{algorithm}[t]
\caption{General NMF Algorithms: Two-Block Coordinate Descent Method \cite{qian2020fast}}
\begin{algorithmic}[1]

\REQUIRE Data matrix $\matr{V}$, number of iteration $n$
\ENSURE  Dictionary matrix $\matr{W}$
\\ \textit{Initialisation} : $\matr{W}_0$, $\matr{H}_0$

\FOR {$t = 0$ to $n-1$}

\STATE $\matr{W}_{t+1} \leftarrow $ update($\matr{V}$,$\matr{W}_{t}$,$\matr{H}_{t}$)

\STATE  $\matr{H}_{t+1} \leftarrow $ update($\matr{V}$,$\matr{W}_{t+1}$,$\matr{H}_{t}$)

\ENDFOR
\RETURN Optimized dictionary matrix $\matr{W}$ 

\end{algorithmic} 
\label{alg:alg_two_block} 
\end{algorithm}

Most of the algorithms discussed in the Section \ref{sec:introduction} follow the two-block coordinate descent framework shown in the Algorithm \ref{alg:alg_two_block}. First, the dictionary matrix $\matr{W}$ is updated, while keeping the coefficient $\matr{H}$ constant. Then, the updated dictionary matrix $\matr{W}$ is used to update the coefficient $\matr{H}$. The process continues until some convergence criteria is met. In our work, we use a robust non-negative matrix factorization solver considering the outliers~\cite{zhao2016online} based on projected gradient descent (PGD). We intend to decompose the noisy data matrix $\matr{V}$ as $\matr{V} \approx \matr{WH} + \matr{R}$, where $\matr{W}$ and $\matr{H}$ are defined as before. The matrix  $\matr{R} = [\vect{r}_{1},\vect{r}_{2}....,\vect{r}_{N}] \in \mathcal{R}^{D\times N}$ is a matrix containing the outliers of the data. Thus, the NMF optimization problem is reformulated as
\begin{equation} \label{update_optimization}
      \min_{\matr{W} \in \mathcal{C}, \matr{H}\geqslant 0,\matr{R} \in \mathcal{Q}} \frac{1}{N} \big( \frac{1}{2}     \norm{\matr{V}-\matr{W} \matr {H}-\matr{R}}_F^2 + \lambda\norm{\matr{R}}_{1,1} \big),
\end{equation}
where $\mathcal{C} \subseteq \mathcal{R}^{D\times K}$ is the constraint set for updating $W$, $\mathcal{Q}$ is the feasible set for $\matr{R}$ and $\lambda \geqslant 0$ is the regularization parameter. Intuitively, the outlier matrix $\matr{R}$ is sparse in nature and contains a smaller values compared to the noise-free original data entries. \textcolor{blue}{The sparsity of the outlier $\matr{R}$ is enforced by the choice of $\mathcal{L}_{1,1}$-norm regularization, as discussed in \cite{zhao2016online}. The level of sparsity is controlled by the hyper-parameter $\lambda$. In line with the work of \cite{zhao2016online}, we used the same hyper-parameters in most cases for our experiments. Evidently, a grid search can be performed to select optimal hyper-parameters for our proposed differentially-private NMF. Additionally, one can compare the performance of NMF with outliers with $\mathcal{L}_{2}$ or other norms, but we defer that for future work.}

Note that the robust NMF algorithm can not guarantee exact recovery of the original data matrix $\matr{V}$, as the loss function \eqref{update_optimization} is not convex in nature \cite{zhao2016online}. However, it can be shown empirically that the estimated basis matrix $\matr{\Whatv}$ can be meaningful, and the difference between the data matrix $\matr{V}$ and the new reconstructed matrix $\matr{\Whatv\Hhatv}$ is negligible and sparse \cite{zhao2016online,shen2014robust,fevotte2015nonlinear}. Now, following the two-block coordinate descent method mentioned in the Algorithm \ref{alg:alg_two_block}, we reformulate our optimization steps as follows:

\begin{enumerate}
    \item Update the coefficient matrix $\matr{H}_t$ and outlier matrix $\matr{R}_t$ based on the fixed dictionary matrix $\matr{W}_{t-1}$. Here the optimization can be done as \begin{equation}\label{hr}
        (\matr{H}_t,\matr{R}_t)= \argmin_{\matr{H}\geqslant 0,\matr{R} \in \mathcal{Q}} \hspace{5pt} L( \matr{V},\matr{W}_{t-1},\matr{H},\matr{R}),
    \end{equation} 
    where the loss function $L$ is
    \begin{align*}
        L(\matr{V},\matr{W},\matr{H},\matr{R}) &\triangleq \frac{1}{N} ( \frac{1}{2} \norm{\matr{V}- \matr{WH}- \matr{R}}_F^2\\
        & +\lambda \norm{\matr{R}}_{1}) \numberthis \label{loss_org}.
    \end{align*}
    
    Here constraint set $\mathcal{Q} \triangleq \set{\vect{r} \in \mathds{R}^D \text{ | } \norm{ \vect{r}}_{\infty} \leqslant M}$ keeps the entries of outlier matrix $\matr{R}$ uniformly bounded. The value of $M$ depends on the data set and noise distribution. For example, in the gray scale image data with $2^b-1$ levels in each pixel, we can set $M=2^b-1$ where b is the number of bits to present the pixel value.      
    
    \item After optimization with respect to $\matr{H}_t$ and $\matr{R_t}$, update $\matr{W}_t$ using the same loss function \eqref{loss_org}:
    \begin{equation}\label{w}
        \matr{W}_t= \argmin_{\matr{W} \in \mathcal{C}} \hspace{5pt} L(\matr{V},\matr{W},\matr{H}_{t},\matr{R}_{t}).
    \end{equation} 
    
    Here, the set $\mathcal{C}$ constrains the columns of dictionary matrix $\matr{W}$ into a unit (non-negative) $\ell_2$ ball to keep the matrix entries bounded \cite{lin2007convergence,lin2007projected}. \\
\end{enumerate}

\noindent\textbf{PGD Solver for \eqref{hr}. }Equation \eqref{hr} can be  solved by following the two steps alternatively for a fixed $\matr{W}$. 
\begin{equation} \label{argminH}
    \matr{H}^+ := \argmin_{\matr{H'} \geqslant 0}Q (\matr{H'} | \matr{H}),
\end{equation}
\begin{equation} \label{argminR}
    \matr{R}^+= \argmin_{\matr{R'} \in \mathcal{Q}} \frac{1}{N}\Big ( \frac{1}{2} \norm{\matr{V}- \matr{WH_{t}}- \matr{R'}}_F^2 + \lambda \norm{\matr{R'}}_{1} \Big),
\end{equation}
where
\begin{equation}
    Q(\matr{H'}|\matr{H}) \triangleq q(\matr{H})+<\bigtriangledown q (\matr{H}), \matr{H'}-\matr{H}> + \frac{1}{2\eta N} \norm{\matr{H'}-\matr{H}}_2^2, 
\end{equation}
\begin{equation} \label{q(h)}
    q(\matr{H}) \triangleq \frac{1}{2N} \norm{\matr{V}-\matr{WH}-\matr{R}}_F^2.
\end{equation}
Here, $\eta$ is the fixed step size. Minimizing both \eqref{argminH} and \eqref{argminR} have closed-form solutions. For \eqref{argminH}, the solution can be formulated as following
\begin{equation} \label{h_update}
    \matr{H}^+ := \mathcal{P}_{+}(\matr{H}-\eta_{H} \bigtriangledown q(\matr{H}) ).
\end{equation}
Here, we replace the step size $\eta$ with $\eta_{H}$ to distinguish this from the dictionary matrix $\matr{W}$ update. We use a fixed step size to ease the hyper-parameter setting in the whole iteration process. $\bigtriangledown q(\matr{H})$ is the gradient function derived by doing partial derivative of \eqref{q(h)} with respect to $\matr{H}$.
\begin{equation} \label{Hgradient}
    \triangledown q(\matr{H})=\frac{1}{N} \big(\matr{W}^{\top}\matr{WH}-\matr{W}^{\top}( \matr{V}- \matr{R})\big).
\end{equation}
We use the following matrix properties \cite{petersen2008matrix} to find the expression of $\bigtriangledown q(\matr{H})$: $\norm{\matr{A}}_{2}^2=\tr(\matr{A^\top}\matr{A})$, and $\bigtriangledown \tr(\matr{X}^\top\matr{A}) = \matr{A}$, where the gradient is taken with respect to $\matr{X}$.
For \eqref{argminR}, the solution is straightforward.
\begin{equation}\label{r_update}
    \matr{R}^+=S_{\lambda,M} \big(\matr{V}-\matr{WH^+}\big).
\end{equation}
Here, $S_{\lambda,M}(\matr{X})$ performs element-wise thresholding as:
\begin{equation}
    S_{\lambda,M}(\matr{X})_{ij} := 
    \begin{cases}
       0, & \norm{x_{ij}} < \lambda \\
       
       x_{ij}-\textrm{sgn}(x_{ij})\lambda, & \lambda \leqslant \norm{x_{ij}} \leqslant \lambda+M \\
       
       \textrm{sgn}(x_{ij})M. & \norm{x_{ij}} > \lambda+M

    \end{cases}
\end{equation}
In the $t$-th iteration, we update the matrices $\matr{H}$ and $\matr{R}$ according to \eqref{h_update} and \eqref{r_update}; until some stopping criteria is met~\cite{zhao2016online}.\\

\noindent\textbf{PGD Solver for \eqref{w}. }We can rewrite (\ref{w}) as follows
\begin{equation} \label{WAB}
    \matr{W}_t= \argmin_{\matr{W} \in \mathcal{C}} \frac{1}{2}\tr(\matr{W}^\top\matr{W}\matr{A}_t)-\tr(\matr{W}^\top\matr{B}_{t}),
\end{equation}
where $\matr{A}_t \triangleq \frac{1}{N}\matr{H}_{t}\matr{H}_{t}^\top$, and $\matr{B}_t \triangleq \frac{1}{N} (\matr{V-R_{t}})\matr{H}_{t}^\top$. To calculate the gradient value, we define new function $f_{W}(\matr{W})$.
\begin{equation} \label{fw}
    f_{W}(\matr{W})= \frac{1}{2}\tr(\matr{W}^\top\matr{W}\matr{A}_t)-tr(\matr{W}^\top\matr{B}_{t}).
\end{equation}
Taking partial derivative of \eqref{fw} with respect to $\matr{W}$, we find the following expression.
\begin{equation}
\label{w_gradient}
    \triangledown f_{W}(\matr{W})=\matr{WA}-\matr{B}.
\end{equation}
We use the aforementioned matrix property~\cite{petersen2008matrix} to derive the expression \eqref{w_gradient}. We write the update equation ensuring the constraints of $\matr{W}$.
\begin{equation} \label{w_update}
    \matr{W}^+=\mathcal{P}_{\mathcal{C}}( \matr{W}-\eta_{W}\triangledown f_{W}(\matr{W})).
\end{equation}
Here, $\eta_{W}$ is the step size to update $\matr{W}$. As in \eqref{h_update}, we use a fixed step size. The constraint projection function keeps the columns of the dictionary matrix $\matr{W}$ in unit $\ell_2$ ball. In \eqref{w_update}, each column is being updated as following
\begin{equation}
    \vect{w}_{n}^{+}:=\frac{\mathcal{P}_{+}\big(\vect{w}_{n}-\eta_{W}\triangledown f_{W}(\vect{w}_{n})\big)}{\max\Big(1,\norm{\mathcal{P}_{+}\big(\vect{w}_{n}-\eta_{W}\triangledown f_{W}(\vect{w}_{n})\big)_{}}_{2}\Big)}, \forall n \in [K].
\end{equation}
Similarly as the PGD Solver for \eqref{hr}, we use \eqref{w_update} to update the dictionary matrix $\matr{W}$. The steps stated above are followed until we reach the optimum loss point of \eqref{update_optimization}. In the Section \ref{experimental_results}, we will discuss hyper-parameter tuning and matrix initialization methods for better optimization. 

Now that we discussed estimating the dictionary matrix $\matr{W}$ from the data matrix $\matr{V}$ without any privacy constraint, in the next section we will show how to preserve and control privacy leakage for the NMF.

\section {Proposed method} \label{proposed_method}

\subsection{Separate Private and Non-private Training Nodes} 

\begin{figure}[t]
    \centering
    \includegraphics[scale=0.25]{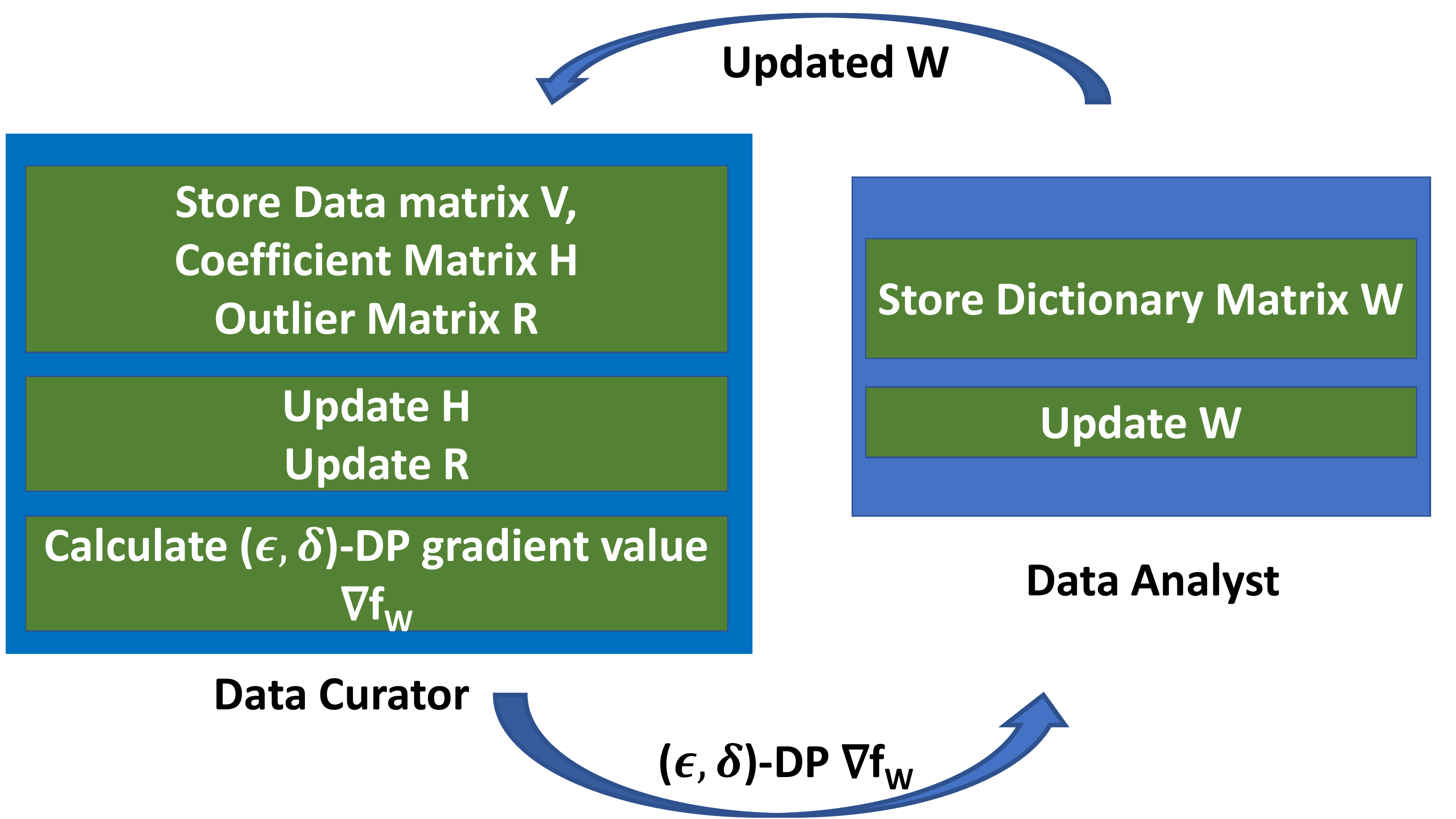}
    \caption{Schematic diagram of privacy-preserving NMF}
    \label{fig:basic_archi}
\end{figure}

As mentioned earlier, we are interested in estimating the dictionary matrix $\matr{W}$, which captures the meaningful population feature. The decomposition also produces two more matrices $\matr{H}$ and $\matr{R}$.
All these three matrices are estimated from the user data $\matr{V}$; however, we only share the dictionary matrix $\matr{W}$ in public. Therefore, we need to estimate the matrix $\matr{W}$ satisfying differential privacy. Intuitively, we need to process the user-sensitive data and population feature data separately. Fig.\ref{fig:basic_archi} shows the basic system diagram which serves the purposes. Here, there are two data processing centers: one is the ``Data Curator'', which holds the sensitive data: data matrix $\matr{V}$ and the coefficient matrix $\matr{H}$ and as well as the outlier matrix $\matr{R}$. It updates the matrices $\matr{H}$ and $\matr{R}$ as mentioned in the Section \ref{nmf_problem_formulation}. Next, it calculates the gradient value $f_{W}(\matr{W})$ and adds white noise with variance depends on the privacy budget and $\mathcal{L}_2$ sensitivity of the gradient function $f_{W}(\matr{W})$. The other data processing center is the ``Data Analyst'' center, where the noisy gradient values are passed to. This center then updates the dictionary matrix $\matr{W}$ with the received noisy gradient value and passes the updated $\matr{W}$ to the Data Curator center. This cycle continues until the stopping criteria is met. At the end, we get a differentially-private dictionary matrix $\matr{W}_\textrm{private}$ at the Data Analyst center. 

\subsection{Derivations Related to Estimating Differentially-Private $\matr{W}$}
\begin{table}[t]
    \centering
    \begin{tabular}{c c c c}
         \hline
         Dataset & $N$ & $D$ & $K$ \\ 
         \hline
         Guardian News Articles & $4551$ & $10285$ & $8$ \\
         UCI-news-aggregator & $3000$ & $93$ & $7$ \\
         RCV1 & $9625$ & $2999$ & $4$ \\
         TDT2 & $9394$ & $3677$ & $30$ \\
         YaleB & $2414$ & $1024$ & $38$ \\
         CBCL & $2429$ & $361$ & $50$ \\
         \hline
               
    \end{tabular}
    \caption{Summary of data sets}
    \label{tab:dataset_size}
\end{table}
In this section, we will show the necessary proof and derivations related to the calculation of $\mathcal{L}_2$ sensitivity for estimating the dictionary matrix $\matr{W}$ satisfying differential privacy. 
As mentioned earlier, we use the Gaussian mechanism to deploy the DP mechanism in our non-private algorithm. Prior to adding noise, we need to calculate the $\mathcal{L}_2$ sensitivity of the function of whose output we will make randomize. In our implementation, it is $\bigtriangledown f_{W}(\matr{W})$ -- the gradient function for updating $\matr{W}$ matrix: $\triangledown f_{W}(\matr{W})=\matr{WA}-\matr{B}$.
As $\triangledown f_{W}(\matr{W})$ depends on the statistics matrix $\matr{A}$ and $\matr{B}$, we need to calculate their $\mathcal{L}_2$ sensitivity separately. Let us first calculate the $\mathcal{L}_2$ sensitivity of matrix $\matr{A}=\frac{1}{N} \matr{H}\matr{H}^\top$.
Consider  two neighboring data sets, the corresponding coefficient matrices are $\matr{H}$ and $\matr{H}'$. By definition, they differ by only one user data. In our calculation, We consider that the difference is at the last $N$-th column entries. We derive the $\mathcal{L}_2$ sensitivity $\Updelta_{A}$ following the definition.  
\begin{align*}
    \Updelta_{A} &= \max \frac{1}{N}\norm{\matr{H} \matr{H}^\top-\matr{H}' \matr{H}'^\top}_F \\
    &=\max \frac{1}{N}\norm{\vect{h}_N \vect{h}^\top_N -\vect{h}'_N \vect{h}'^\top_N}_F \\
    &\leqslant \max \frac{1}{N}\big(\norm{\vect{h}_N \vect{h}_N^\top}_F+\norm{\vect{h}'_N \vect{h}'^\top_N}_F\big) \\
    &\leqslant \max \frac{1}{N} \big( \norm{\vect{h}_N}_2\norm{\vect{h}_N}^\top_2 + \norm{\vect{h}'_N}_2\norm{\vect{h}'_N}^\top_2 \big)\\
    &=\max \frac{1}{N} \big (\norm{\vect{h}_N}_2^2+\norm{\vect{h}'_N}_2^2\big) \\
    &=\frac{2}{N}\times \text{(max $\mathcal{L}_2$ norm of column of $\matr{H}$)}^2 \numberthis \label{A_sensitivity}, 
\end{align*}
where we have used the triangle inequality and $\norm{\vect{ab}}\leqslant \norm{\vect{a}}\norm{\vect{b}}$. To get a bounded value in \eqref{A_sensitivity}, we need to define the max $\mathcal{L}_2$ norm for the column of $\matr{H}$. One way to do that is by normalizing each column of $\matr{H}$, and therefore, $\Updelta_{A} = \frac{2}{N}$. Next, we will calculate the $\mathcal{L}_2$ sensitivity of $\matr{B} = \frac{1}{N}(\matr{V-R})\matr{H}^\top$.
Following the definition, we consider two neighboring data matrices $\matr{V}$ and $\matr{V}'$ and their corresponding neighboring coefficient matrices $\matr{H}$ and $\matr{H}'$. The details calculation of calculating $\mathcal{L}_2$ sensitivity $\Updelta_{B}$ is given as follows:  
\begin{align*}
    \Updelta_{B} &= \max \frac{1}{N}\norm{(\matr{V}-\matr{R}) \matr{H}^\top-(\matr{V}'-\matr{R}) \matr{H}'^\top}_F \\
    &=\max \frac{1}{N}\norm{(\vect{v}_N-\vect{r}_N) \vect{h}^\top_N-(\vect{v}_N'-\vect{r}_N) \vect{h}'^\top_N}_F \\
    &\leqslant \max \frac{1}{N}\big(\norm{(\vect{v}_N-\vect{r}_N) \vect{h}^\top_N}_F+\norm{(\vect{v}_N'-\vect{r}_N)\vect{h}'^\top_N}_F\big) \\
    &\leqslant \max \frac{1}{N} \big( \norm{(\vect{v}_N-\vect{r}_N)}_2\norm{\vect{h}_N}^\top_2\\ &+\norm{(\vect{v}_N'-\vect{r}_N)}_2\norm{\vect{h}'_N}^\top_2 \big) \\\
    &=\max \frac{1}{N} \big (\norm{(\vect{v}_N-\vect{r}_N)}_2+\norm{(\vect{v}_N'-\vect{r}_N)}_2\big)\\
    &=\max \frac{2}{N} \norm{(\vect{v}_N-\vect{r}_N)}_2 \leqslant \max \frac{2}{N}  (\norm{\vect{v}_N}+\norm{\vect{r}_N})\numberthis \label {B_sensitivity},
\end{align*}
where the second-last equality follows from $\max_{\forall n}\norm{\vect{h}_n}_2=1$, and the last inequality follows from $\norm{\matr{a}-\matr{b}}\leqslant \norm{\matr{a}}+\norm{\matr{b}}$. To get a constant $\mathcal{L}_2$ sensitivity value in \eqref{B_sensitivity}, we can normalize the columns of $\matr{V}$ and $\matr{R}$ to have unit $\mathcal{L}_2$-norm. Thus, we have $\Updelta_{B} = \frac{4}{N}$. Note that, if we do not model the outliers explicitly, we would have $\Updelta_{B}=\frac{2}{N}$. 

Now, as we have computed the $\mathcal{L}_2$ sensitivities $\Updelta_{A}$ and $\Updelta_{B}$, we can generate noise perturbed statistics $\overline{\matr{A}},\overline{\matr{B}}$ following the Gaussian mechanism~\cite{dwork2014algorithmic}. Using these values, we can compute the noisy gradient $\overline{\triangledown f_{W}(\matr{W})}$ and update our dictionary matrix $\matr{W}$. At the end of optimization, we will obtain the differentially private dictionary matrix $\matr{W}_\textrm{private}$. The detailed step-by-step description of our proposed method is summarized in Algorithm~\ref{alg:nmf_private_outlier}. 

\begin{algorithm}[t]
\caption{Privacy Preserving NMF with Outliers}
\begin{algorithmic}[1]

\REQUIRE Data matrix $\matr{V}$, step size $\eta_H$ and $\eta_W$, maximum number of iterations $T$, sensitivities $\Updelta_{A}$ and $\Updelta_{B}$, privacy parameters $\epsilon_t,\ \delta$, regularization parameter $\lambda$
\ENSURE Differentially private dictionary matrix $\matr{W}_\textrm{private}$
\\ \textit{Initialisation} : $\matr{W}_0$, $\matr{H}_0$, $\matr{R}_0$.

\FOR {$t = 1$ to $T$} 

\STATE \algorithmiccomment{/* At Data Curator */} Learn $\matr{H}_t$ and $\matr{R}_t$ based on $\matr{W}_{t-1}$ \\

$(\matr{H}_t,\matr{R}_t) := \argmin_{\matr{H}\geqslant 0,\matr{R} \in \mathcal{Q}} \hspace{5pt} L( \matr{V},\matr{W}_{t-1},\matr{H},\matr{R})$,\\

 where matrices update as follow\\
 $\matr{H}^+ := \mathcal{P}_{+}(\matr{H}-\eta_{H} \bigtriangledown q(\matr{H}) ).$\\
 $\matr{R}^+ :=S_{\lambda,M} \big(\matr{V}-\matr{WH^+}\big).$
 
\vspace{3pt} 
\STATE Calculate the noise perturbed statistics matrices\\

$\overline{\matr{A}_t} := \frac{1}{N} \matr{H}_{t}\matr{H}_{t}^\top + \mathcal{N}(0,\tau^2_A)^{K \times K}$,\\

$\overline{\matr{B}_t} := \frac{1}{N} (\matr{V-R_{t}})\matr{H}_{t}^\top + \mathcal{N}(0,\tau^2_B)^{D \times K}$,\\

where\\

$\tau_A = \frac{\Updelta_A}{\epsilon_t} \sqrt{2\log\frac{1.25}{\delta}}$.\\

$\tau_B = \frac{\Updelta_B}{\epsilon_t} \sqrt{2\log\frac{1.25}{\delta}}$.

\vspace{3pt} 
\STATE Calculate the noisy gradient $\overline{\triangledown f_{W}(\matr{W})}$ with the noise perturbed statistics matrices: $\overline{\triangledown f_{W}(\matr{W})}=\matr{W}_{t-1}\overline{\matr{A}_{t}}-\overline{\matr{B}_{t}}$ \\

\STATE \algorithmiccomment{/* At Data Analyst */} Learn dictionary matrix with the noisy gradient $\overline{\triangledown f_{W}}$ 

$\matr{W}_t := \argmin_{\matr{W} \in \mathcal{C}} \frac{1}{2}\tr(\matr{W}^\top\matr{W}\overline{\matr{A}_t})-\tr(\matr{W}^\top\overline{\matr{B}_{t}})$.\\

where dictionary matrix updates as follow\\

 $\matr{W}^+ :=\mathcal{P}_{\mathcal{C}}( \matr{W}-\eta_{W}\overline{\triangledown f_{W}(\matr{W})})$.
 
\vspace{3pt} 
\ENDFOR

\vspace{3pt} 
\RETURN $\matr{W}_\textrm{private}$ 

\end{algorithmic} 
\label{alg:nmf_private_outlier}
\end{algorithm}

\textcolor{blue}{
\begin{theorem}[Privacy of Algorithm~\ref{alg:nmf_private_outlier}]
\label{theorem1}
Consider Algorithm~\ref{alg:nmf_private_outlier} in the setting of Section~\ref{nmf_problem_formulation}. Then Algorithm\ref{alg:nmf_private_outlier} releases $\big(\frac{T\alpha_\mathrm{opt}}{2}(\frac{\Updelta^2_A}{\tau^2_A}+\frac{\Updelta^2_B}{\tau^2_B})+\frac{\log \frac{1}{\delta}}{\alpha_\mathrm{opt}-1},\delta\big)$ differentially-private basis matrix $\matr{W}_\textrm{private}$ for any $0<\delta <1$ after $T$ iterations, where $\alpha_{ \mathrm{opt}}=1+\sqrt{\frac{2}{T(\frac{\Updelta^2_A}{\sigma^2_A}+\frac{\Updelta^2_B}{\delta^2_B}) } \log \frac{1}{\delta}}$.
\end{theorem}
\begin{proof}
We now analyze Algorithm~\ref{alg:nmf_private_outlier} with R\'enyi Differential Privacy (RDP)~\cite{mironov2017renyi}. Recall that, at each iteration $t$, we compute the noisy estimate of the gradient $\overline{\triangledown f_{W}(\matr{W})}$, using two differentially-private matrices $\overline{\matr{A}_t}$ and $\overline{\matr{B}_t}$. According to Proposition~\ref{rdp_gaussian}, computation of these matrices satisfy $\left(\alpha, \frac{\alpha}{2\left(\frac{\tau_A}{\Updelta_A}\right)^2}\right)$-RDP and $\left(\alpha, \frac{\alpha}{2\left(\frac{\tau_B}{\Updelta_B}\right)^2}\right)$-RDP, respectively. According to Proposition~\ref{comp_rdp}, each step of Algorithm~\ref{alg:nmf_private_outlier} is $(\alpha,\frac{\alpha}{2}(\frac{\Updelta^2_A}{\tau^2_A}+\frac{\Updelta^2_B}{\tau^2_B}))$-RDP. If the number of required iterations for reaching convergence is $T$, then under $T$-fold composition of RDP, the overall algorithm is $(\alpha,\frac{T\alpha}{2}(\frac{\Updelta^2_A}{\tau^2_A}+\frac{\Updelta^2_B}{\tau^2_B}))$-RDP. From Proposition \ref{rdp_dp}, we have that the algorithm satisfies $(\frac{T\alpha}{2}(\frac{\Updelta^2_A}{\tau^2_A}+\frac{\Updelta^2_B}{\tau^2_B}) + \frac{\log\frac{1}{\delta}}{\alpha-1},\delta)$-DP for any $0<\delta <1$. For a given $\delta$, we can compute the optimal $\alpha$ as $\alpha_{ \mathrm{opt}} = 1+\sqrt{\frac{2}{T(\frac{\Updelta^2_A}{\tau^2_A}+\frac{\Updelta^2_B}{\tau^2_B}) } \log \frac{1}{\delta}}$. This $\alpha_{ \mathrm{opt}}$ provides the smallest $\epsilon$, i.e., the smallest privacy risk. Therefore, Algorithm~\ref{alg:nmf_private_outlier} releases a $(\frac{T\alpha_ \mathrm{opt}}{2}(\frac{\Updelta^2_A}{\tau^2_A}+\frac{\Updelta^2_B}{\tau^2_B}) + \frac{\log\frac{1}{\delta}}{\alpha_ \mathrm{opt}-1},\delta)$ differentially-private basis matrix $\matr{W}_\textrm{private}$ for any $0<\delta <1$.
\end{proof}
}



\noindent\textbf{Convergence of Algorithm~\ref{alg:nmf_private_outlier}. }We note that the objective function is
non-increasing under the two update steps (i.e., steps 2 and 5), and the objective function is bounded below. Additionally, the noisy gradient estimate $\overline{\triangledown f_{W}(\matr{W})}$ essentially contains zero mean noise. Although this does not provide guarantees on the excess error, the estimate of the gradient converges in expectation to the true gradient~\cite{bottou1999}. However, if the batch size is too small, the noise can be too high for the algorithm to converge~\cite{song2013stochastic}. Since the total additive noise variance is quite small, the convergence rate is faster. Note that a theoretical analysis of intricate relation between the excess error and the privacy parameters is beyond the scope of the current paper. We refer the reader Bassily et al.~\cite{bassily2014private} for further details.

\section {Experimental results}
\label{experimental_results}
In this section, we compare the utility of private and non-private algorithm. We define the objective value to show the comparison, quantifying how well the algorithm can decompose the matrix. The objective value is calculated using the following formula:
\begin{equation} \label{objective_value}
    \text{Objective Value} = \frac{1}{2N} \norm{\matr{V}^o-\matr{WH}}^2_F  .
\end{equation}
Here, $\matr{V}^o$ is the noise-free clean data set. In our experiments, some of the data sets contain noise, and some of them are not. For a fair comparison, we evaluate only how well the decomposition $\matr{W}$$\matr{H}$ can reconstruct the noise-free data $\matr{V}^o$.

To evaluate our proposed method, we use six real data sets. Among which, four are text data sets, and two are face image data sets. The size of the datasets and the corresponding latent dimensions $K$ are mentioned in the Tab. \ref{tab:dataset_size}. \textcolor{blue}{
For selection of the latent dimension $K$, we mentioned one procedure in Appendix \ref{topic_modeling}, where we calculate the topic coherence score of the Guardian News Articles datasets to select the optimum topic number. For the rest of the data sets, similar procedures can be followed.} In the experiments, we use fixed  $\delta=10^{-5}$ and vary $\epsilon$ to show the effect of the privacy budget. Also, we normalized each data sample of $\matr{V}$ so that it has a unit maximum index value.

\subsection{Hyper-parameter Selection and Initialization}
\label{hyperparameter}

We followed the hyper-parameter settings mentioned as~\cite{zhao2016online} except for the learning rate. The authors in~\cite{zhao2016online} suggested using the same learning rate for updating both the dictionary matrix $\matr{W}$ and the coefficient matrix $\matr{H}$ for ease of parameter tuning. The optimization process requires a different configuration in our privacy-preserving implementation. As discussed in Section \ref{proposed_method}, we add Gaussian noise in the gradient calculation $\triangledown f_{W}(\matr{W})$ for updating $\matr{W}$, wheres matrix $\matr{H}$ updates with its original gradient calculation $\triangledown q(\matr{H})$. This dissimilar nature for updating $\matr{W}$ and $\matr{H}$ motivates us to use different learning rates. \textcolor{blue}{We note that to get faster convergence, we need to choose the learning rates for updating $\matr{W}$ and $\matr{H}$. To that end, we can employ a grid search to find the optimum learning rates for $\matr{H}$ and $\matr{W}$. The time required for such a time search depends on the dataset and the search space. With sub-optimal learning rates, the convergence may be delayed, but the excess error of the proposed algorithm should be approximately the same given that sufficient time is allowed, and small enough learning rate is used.} In our experiments, we performed a grid search, and found that the learning rate for updating $\matr{W}$ should be much lower (about 10,000 to 20,000 times depending on the dataset) compared to that of updating $\matr{H}$. To initialize $\matr{W}_0$ and $\matr{H}_0$, we followed the Non negative Double Singular Value Decomposition (NNDSVD) \cite{boutsidis2008svd} approach, as we found that it performs better than random initialization in minimizing the objective value in \eqref{objective_value}. For sparseness, we initialized $\matr{R}_0$ with all zeros.

\subsection{Text Data Set} 

\begin{figure*}[t]%
\centering
\begin{subfigure}{.95\columnwidth}
\includegraphics[width=\columnwidth]{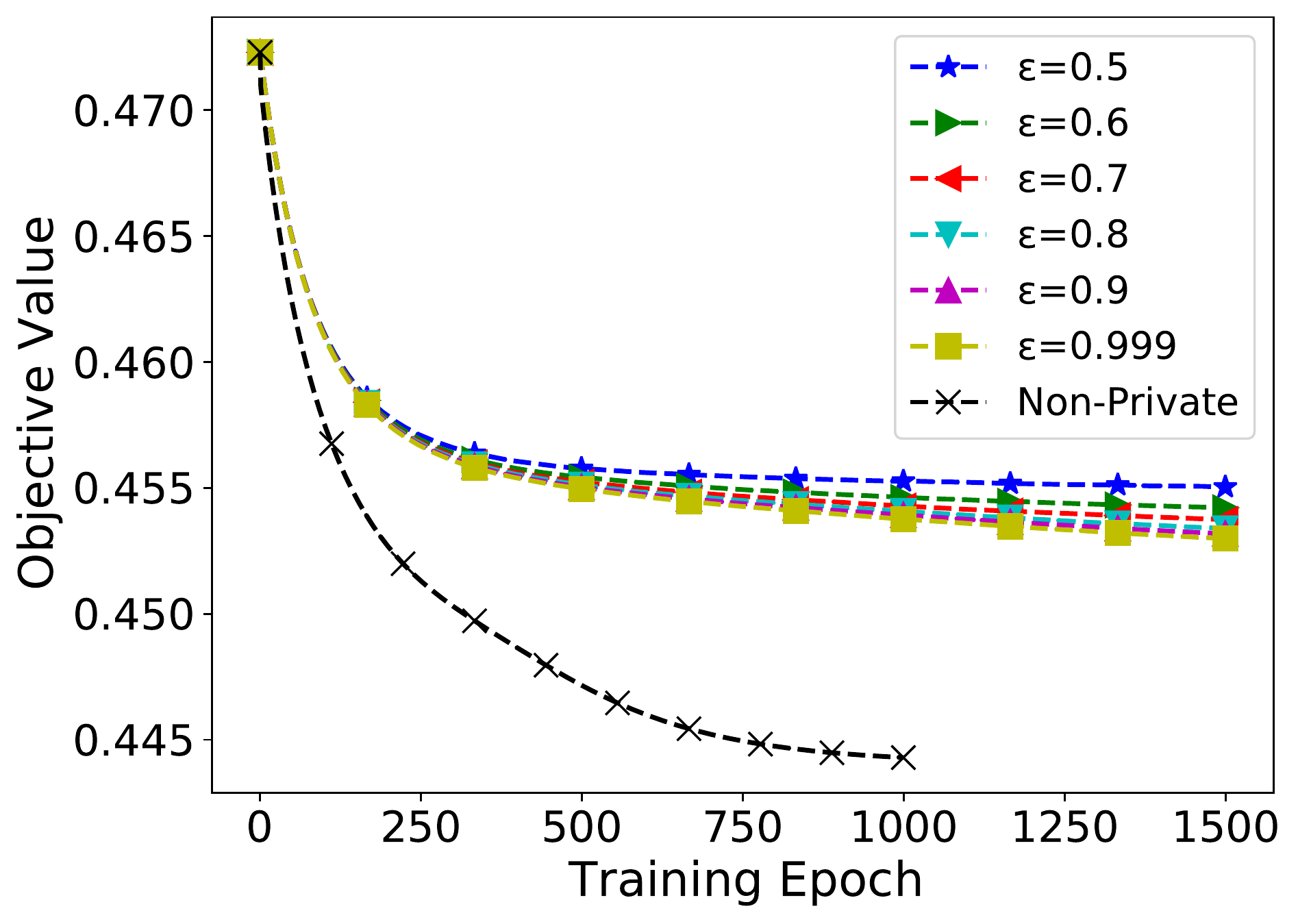}%
\caption{Guardian News Articles}
\label{fig:git_utility}
\end{subfigure}\hfill%
\begin{subfigure}{.95\columnwidth}
\includegraphics[width=\columnwidth]{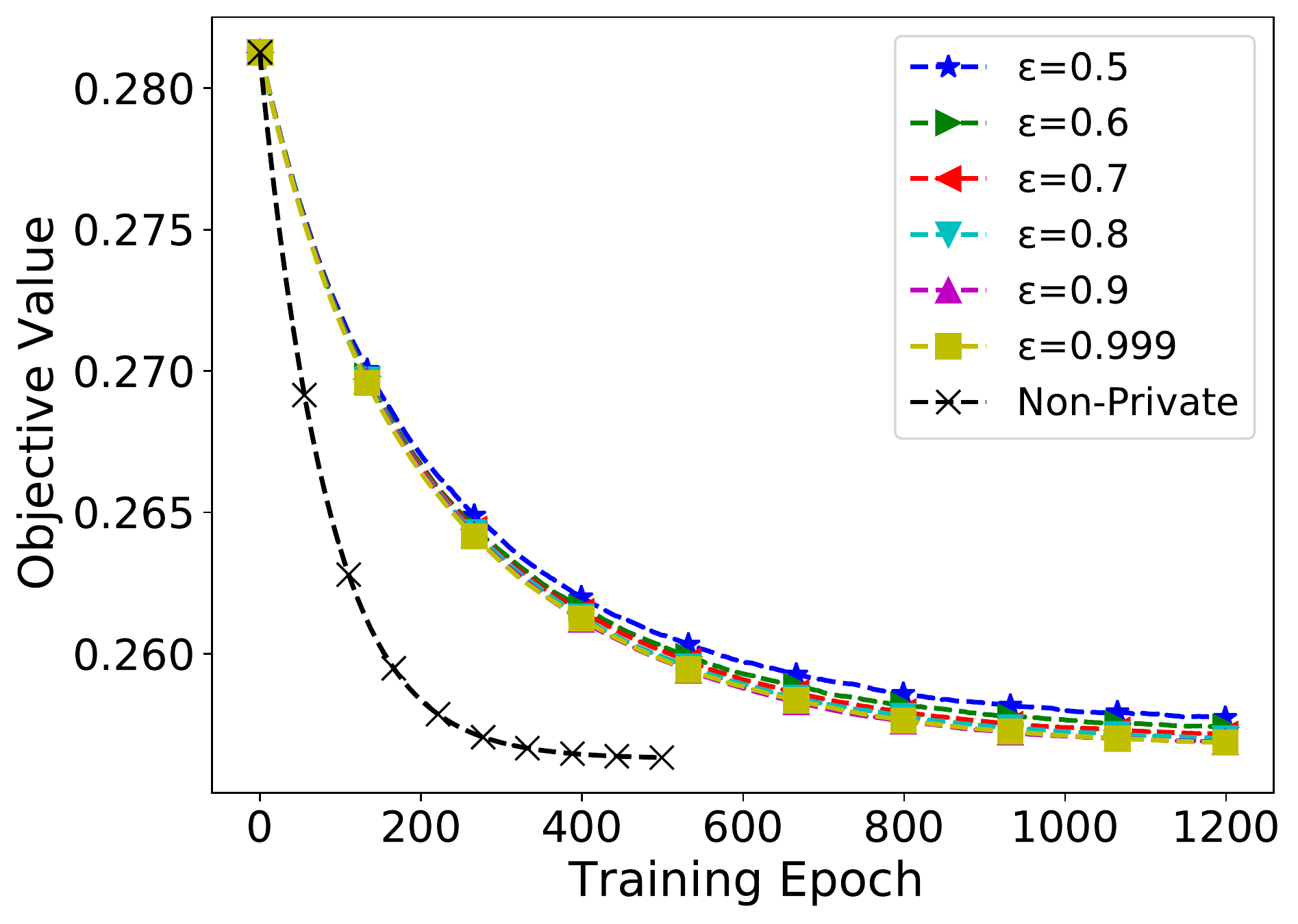}%
\caption{UCI-news-aggregator}
\label{fig:uci_utility}
\end{subfigure}%

\begin{subfigure}{.95\columnwidth}
\includegraphics[width=\columnwidth]{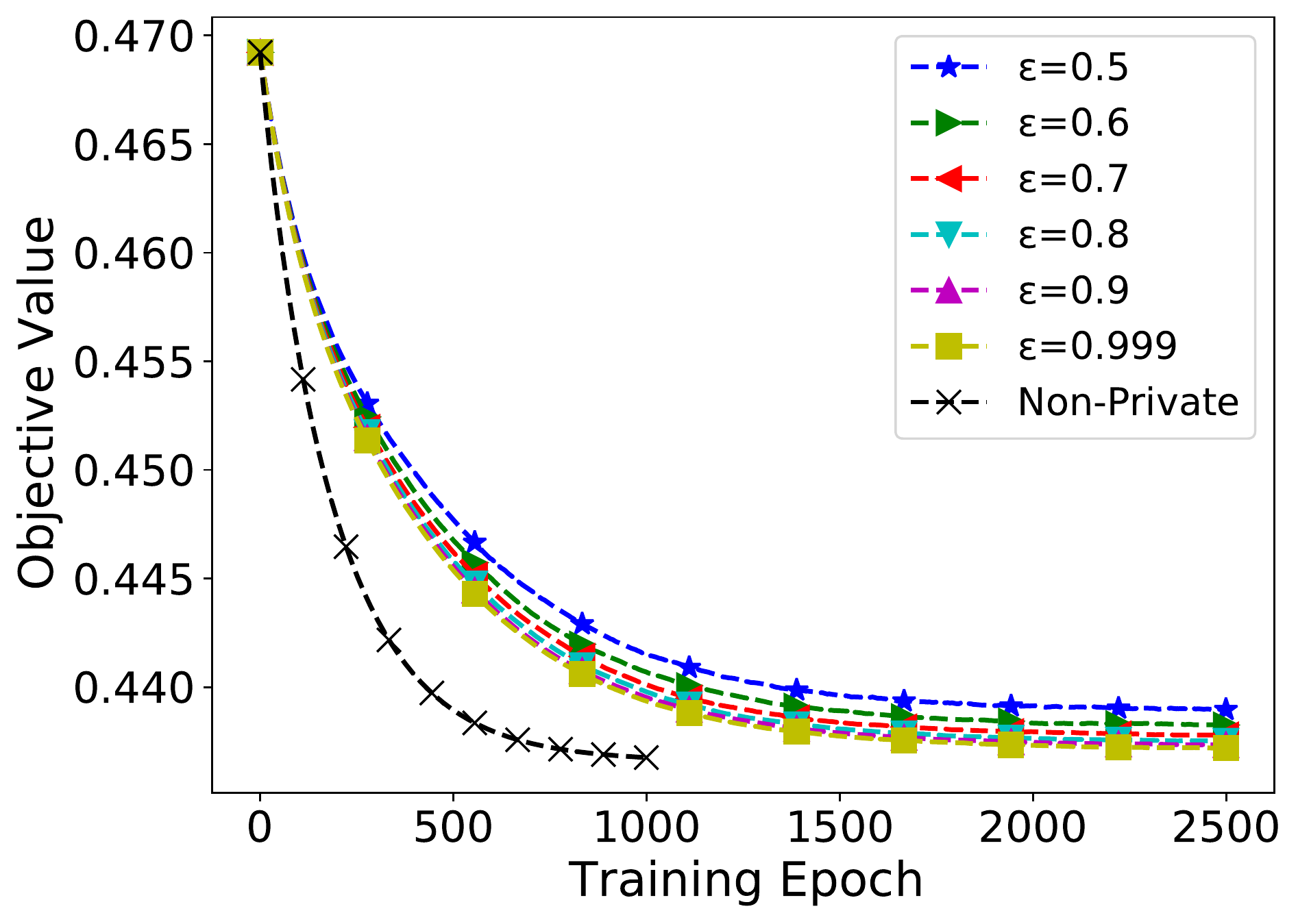}%
\caption{RCV1}
\label{fig:rcv1}
\end{subfigure}\hfill%
\begin{subfigure}{.95\columnwidth}
\includegraphics[width=\columnwidth]{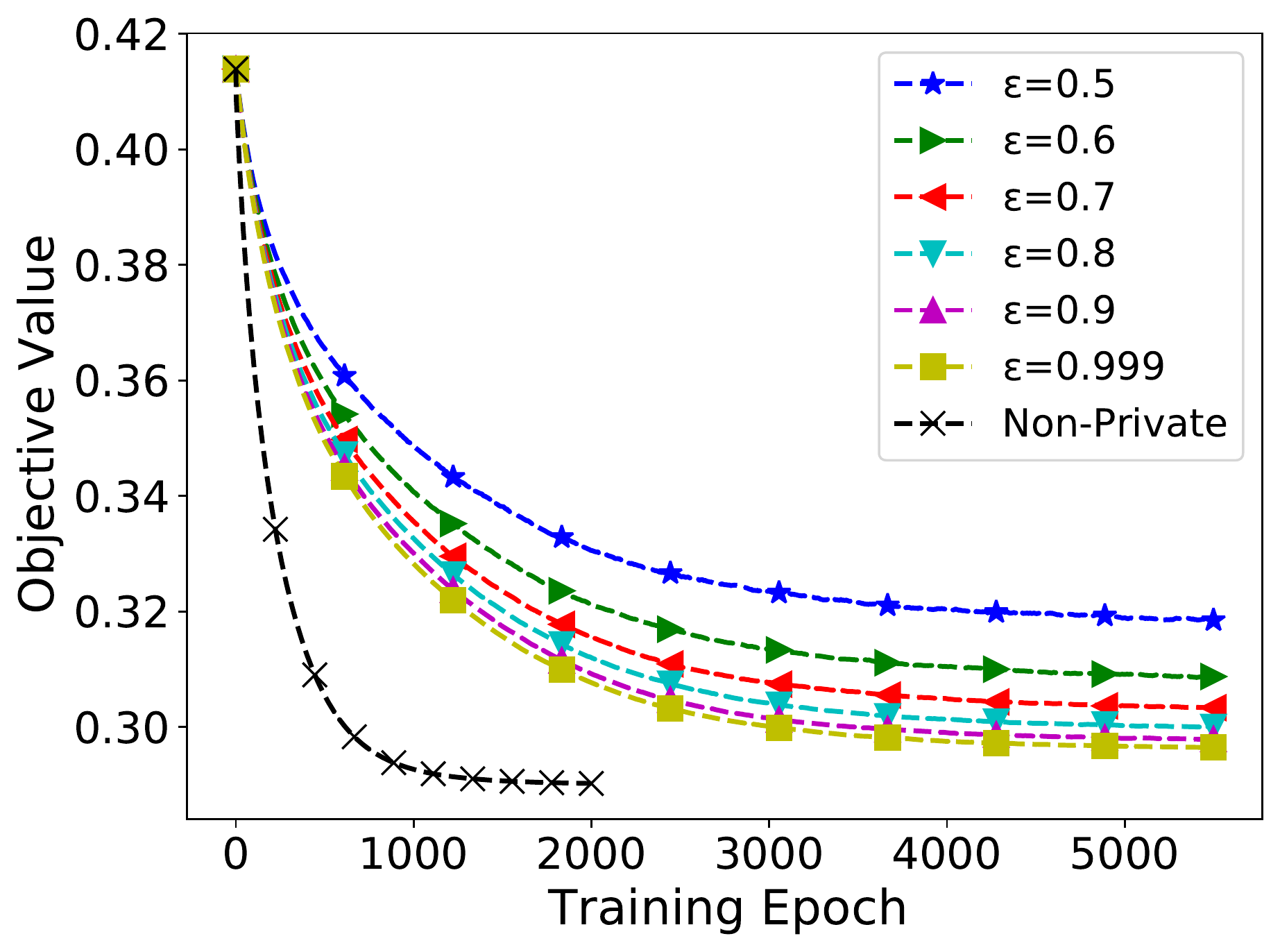}%
\caption{TDT2}
\label{fig:TDT2}
\end{subfigure}%

\caption{Utility Comparison on Text Data Set}
\label{fig:utility_text}
\end{figure*}

\begin{figure*}[t]%
\centering
\begin{subfigure}{.95\columnwidth}
\includegraphics[width=\columnwidth]{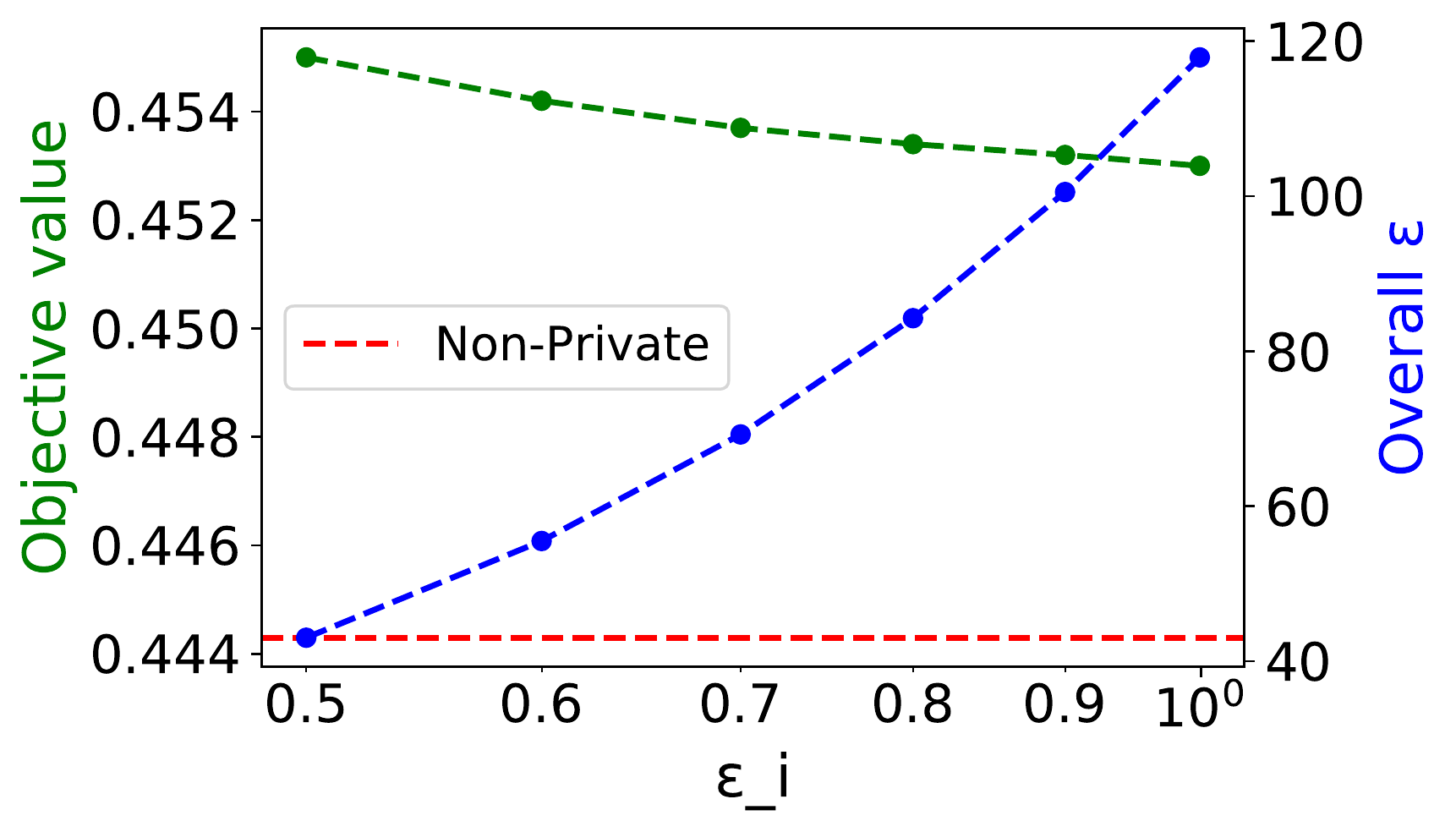}%
\caption{Guardian News Articles}
\label{fig:rdp_gurdian}
\end{subfigure}\hfill%
\begin{subfigure}{.95\columnwidth}
\includegraphics[width=\columnwidth]{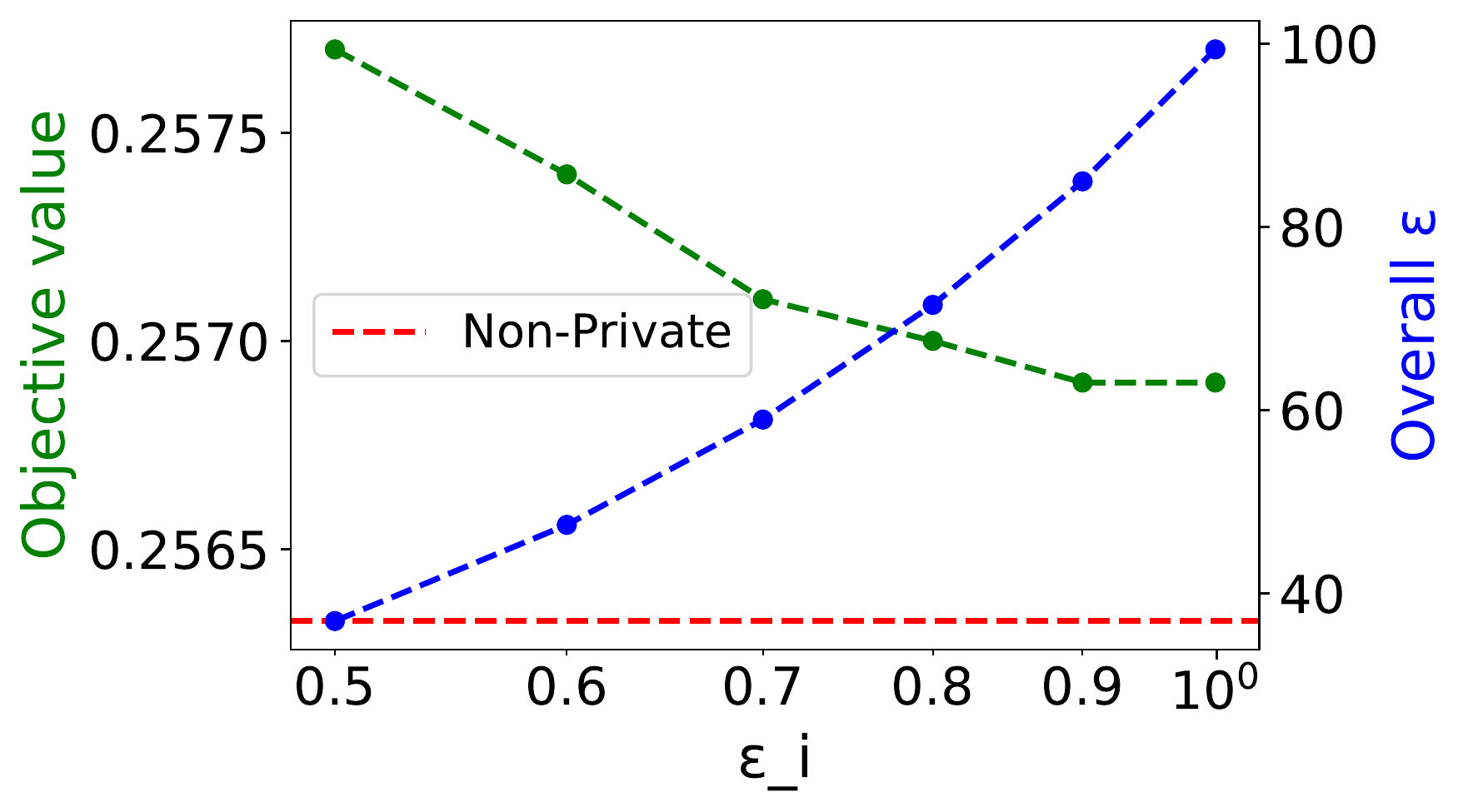}%
\caption{UCI-news-aggregator}
\label{fig:uci_rdp}
\end{subfigure}%

\begin{subfigure}{.95\columnwidth}
\includegraphics[width=\columnwidth]{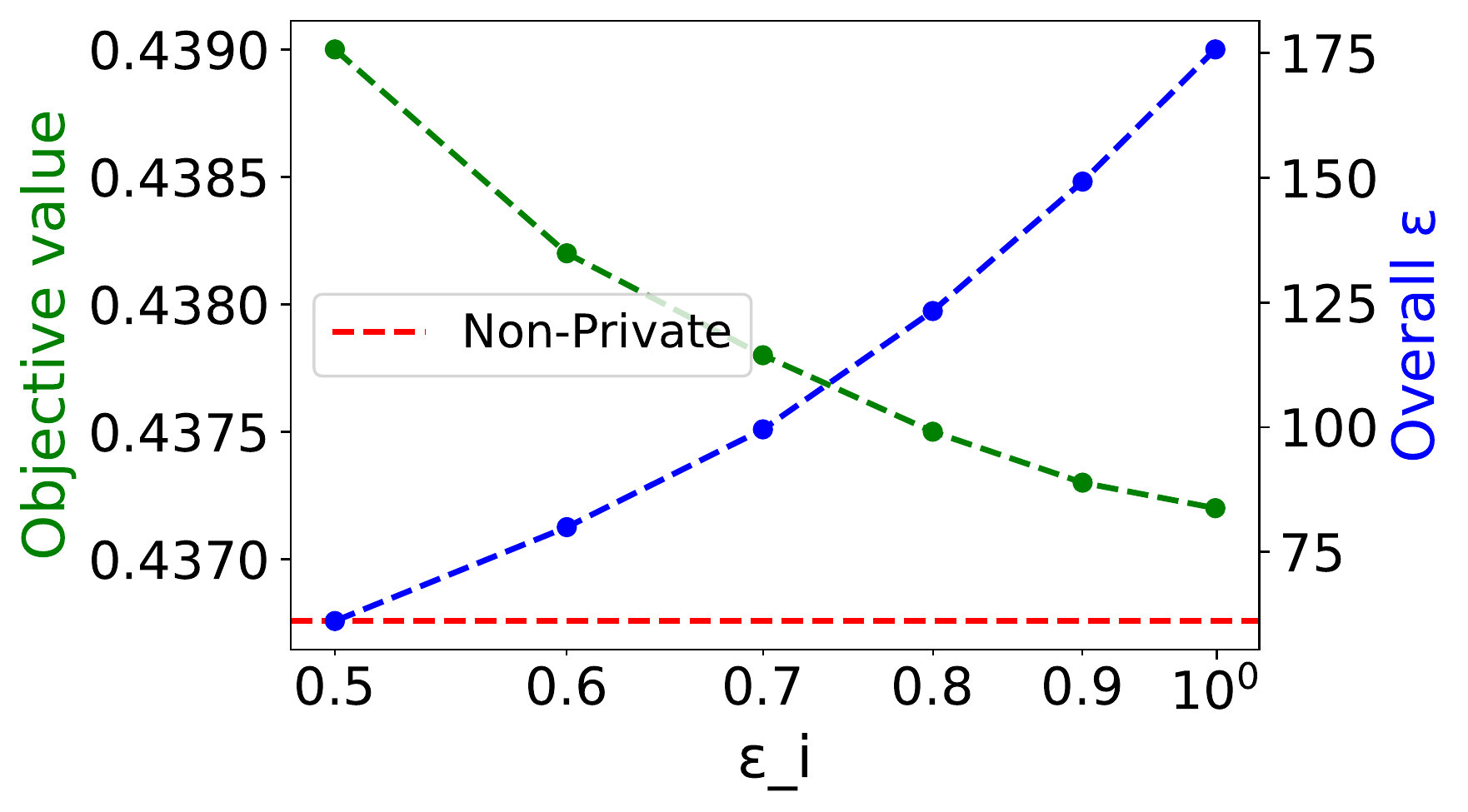}%
\caption{RCV1}
\label{fig:rcv1_rdp}
\end{subfigure}\hfill%
\begin{subfigure}{.95\columnwidth}
\includegraphics[width=\columnwidth]{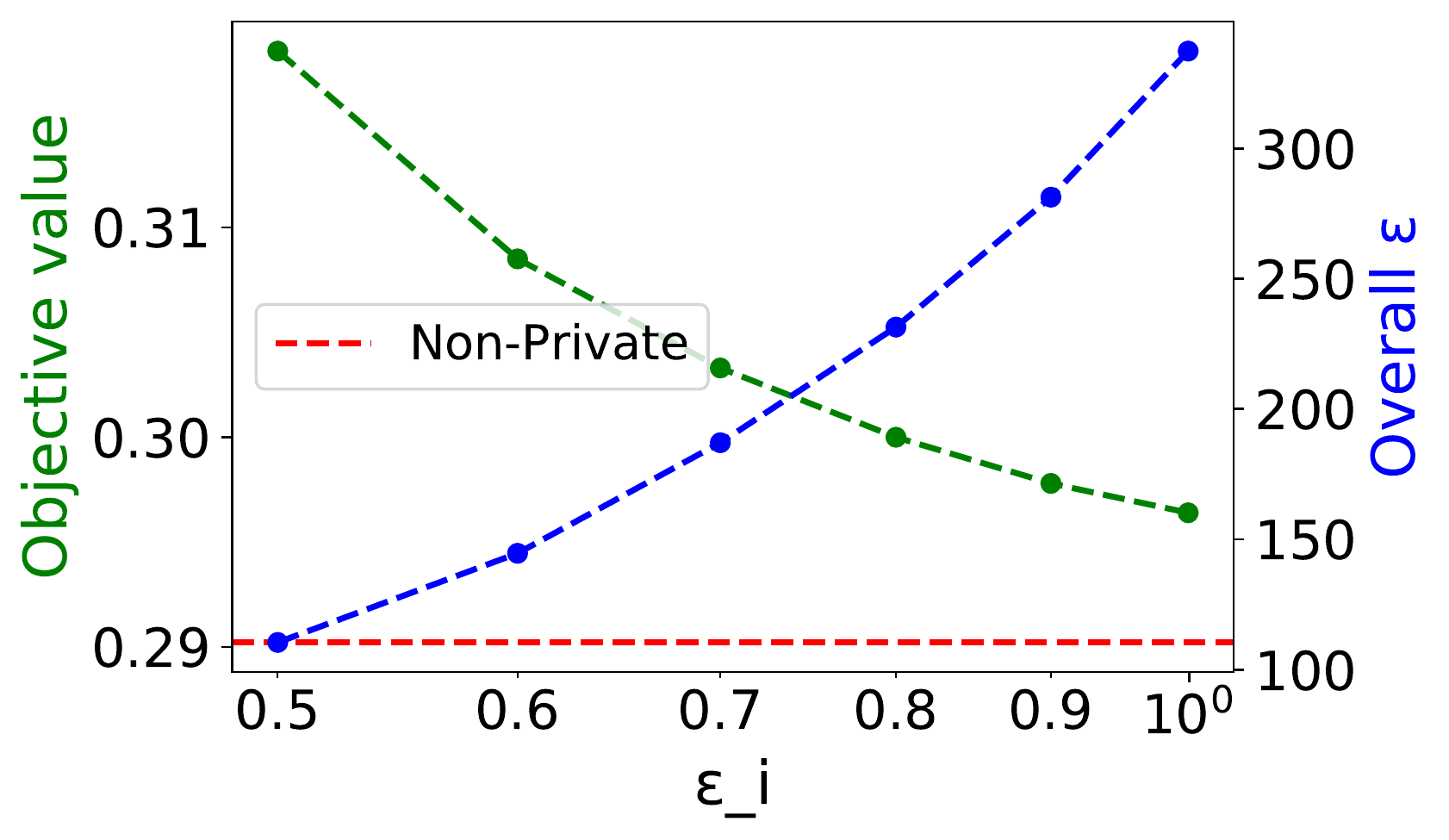}%
\caption{TDT2}
\label{tdt2_rdp}
\end{subfigure}%

\caption{Overall $\epsilon$ and Objective Value on Text Data Set}
\label{fig:rdp_text}
\end{figure*}

\begin{figure*}[t]%
\centering
\begin{subfigure}{.95\columnwidth}
\includegraphics[width=\columnwidth]{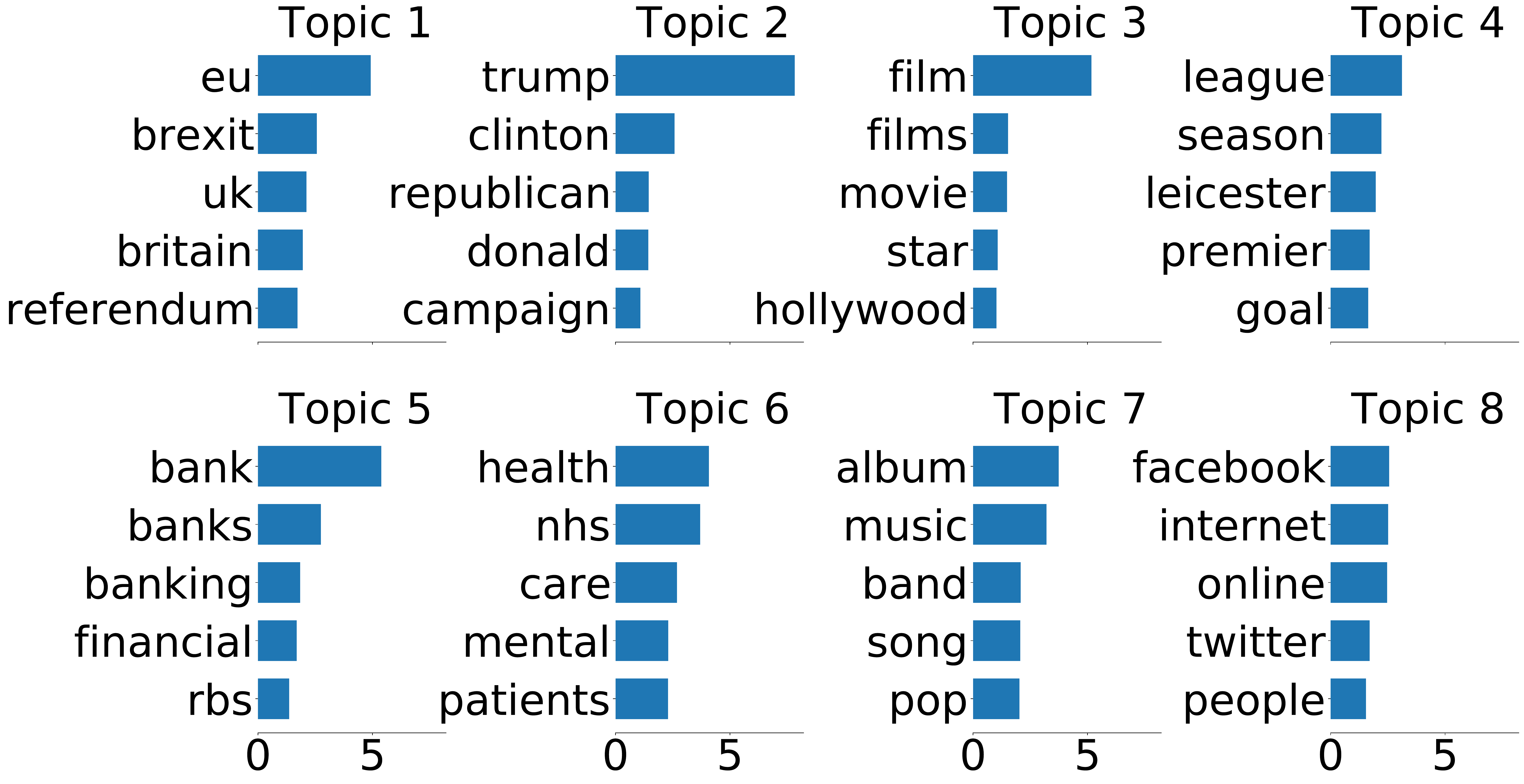}%
\caption{Non-Private Guardian News Articles}
\label{}
\end{subfigure}\hfill%
\begin{subfigure}{.95\columnwidth}
\includegraphics[width=\columnwidth]{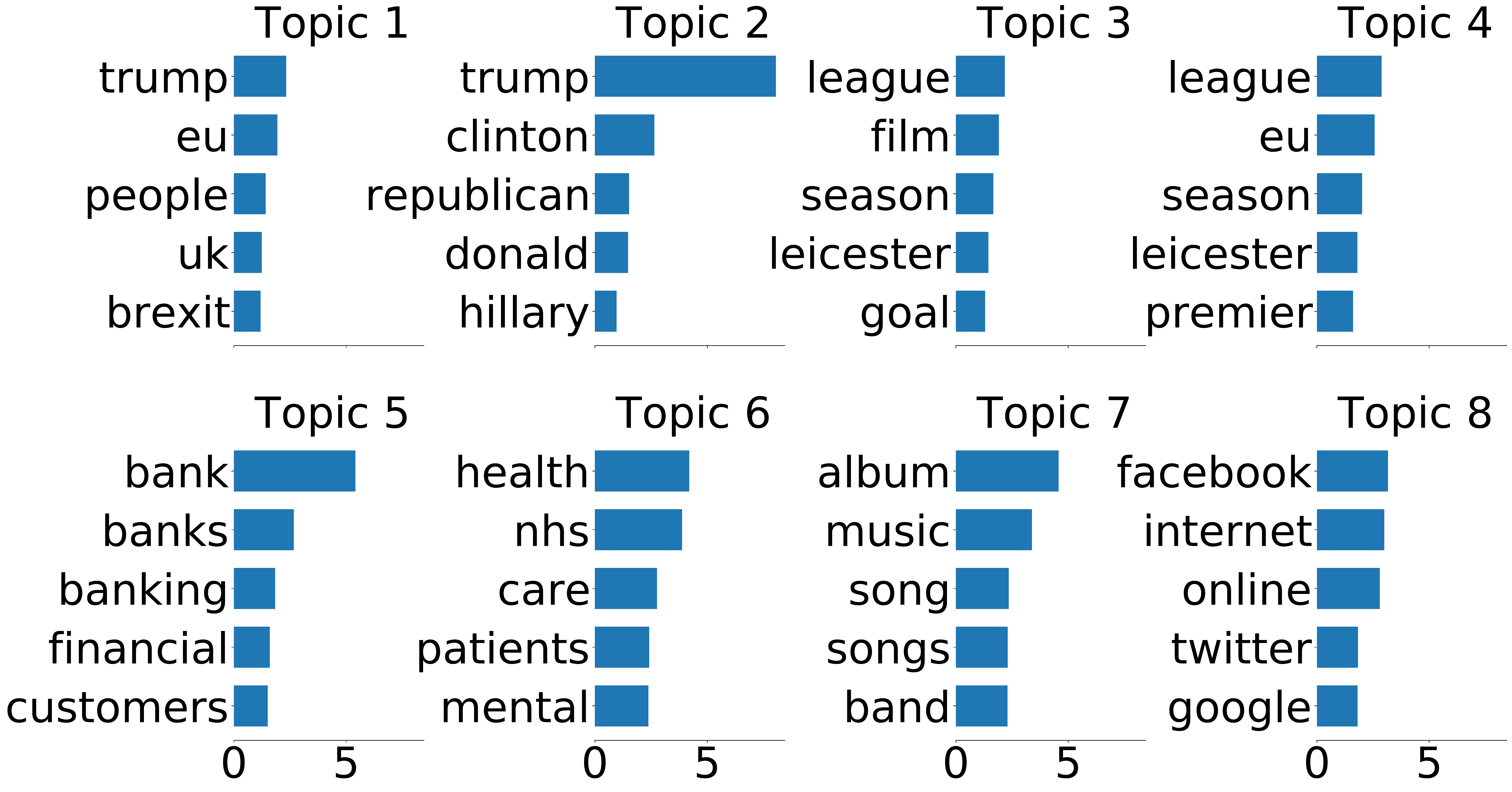}%
\caption{$(\epsilon=0.5,\delta=10^-5)$-DP Guardian News Articles}
\label{}
\end{subfigure}%

\begin{subfigure}{.95\columnwidth}
\includegraphics[width=\columnwidth]{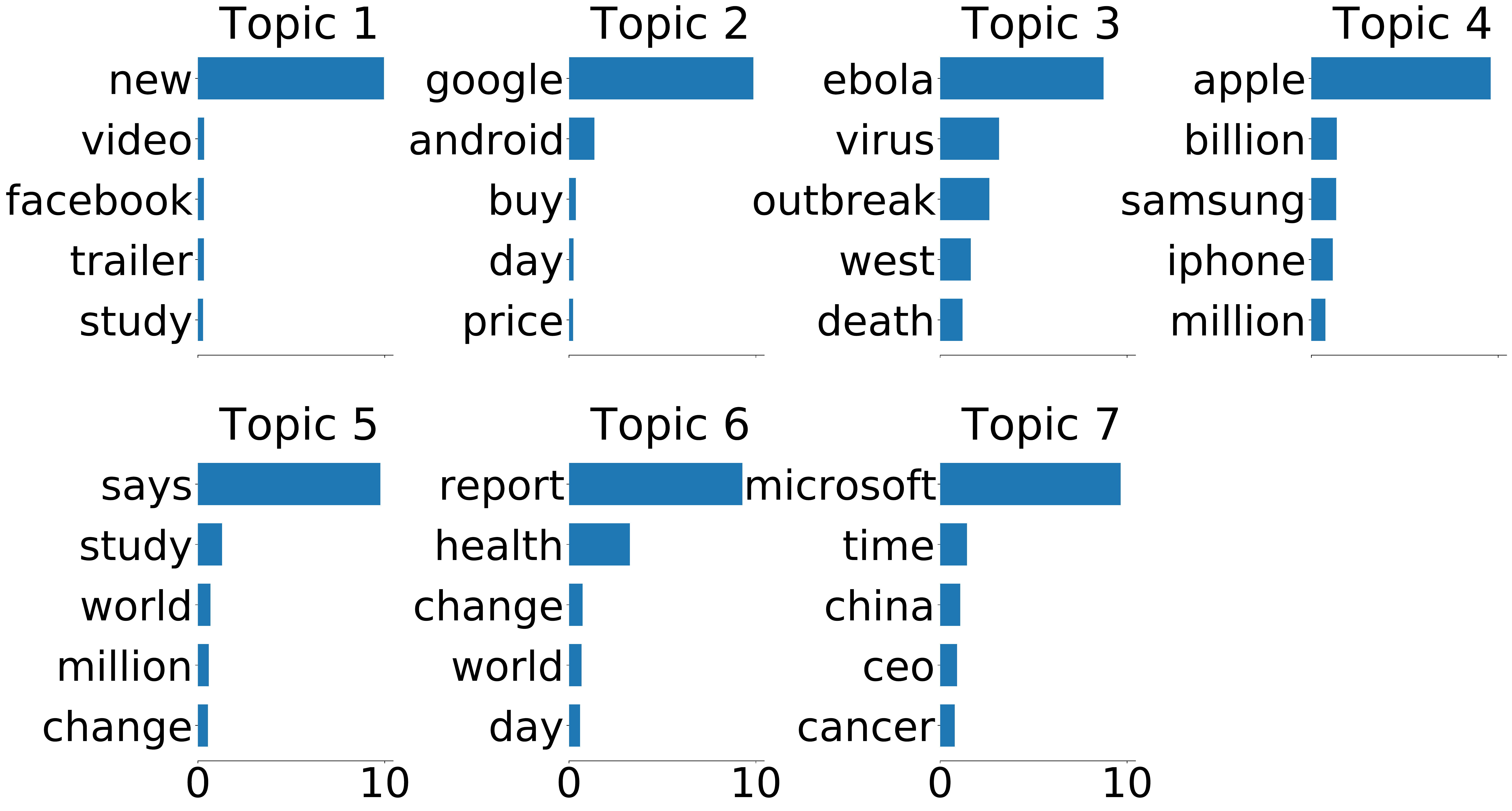}%
\caption{Non-Private UCI-news-aggregator}
\label{}
\end{subfigure}\hfill%
\begin{subfigure}{.95\columnwidth}
\includegraphics[width=\columnwidth]{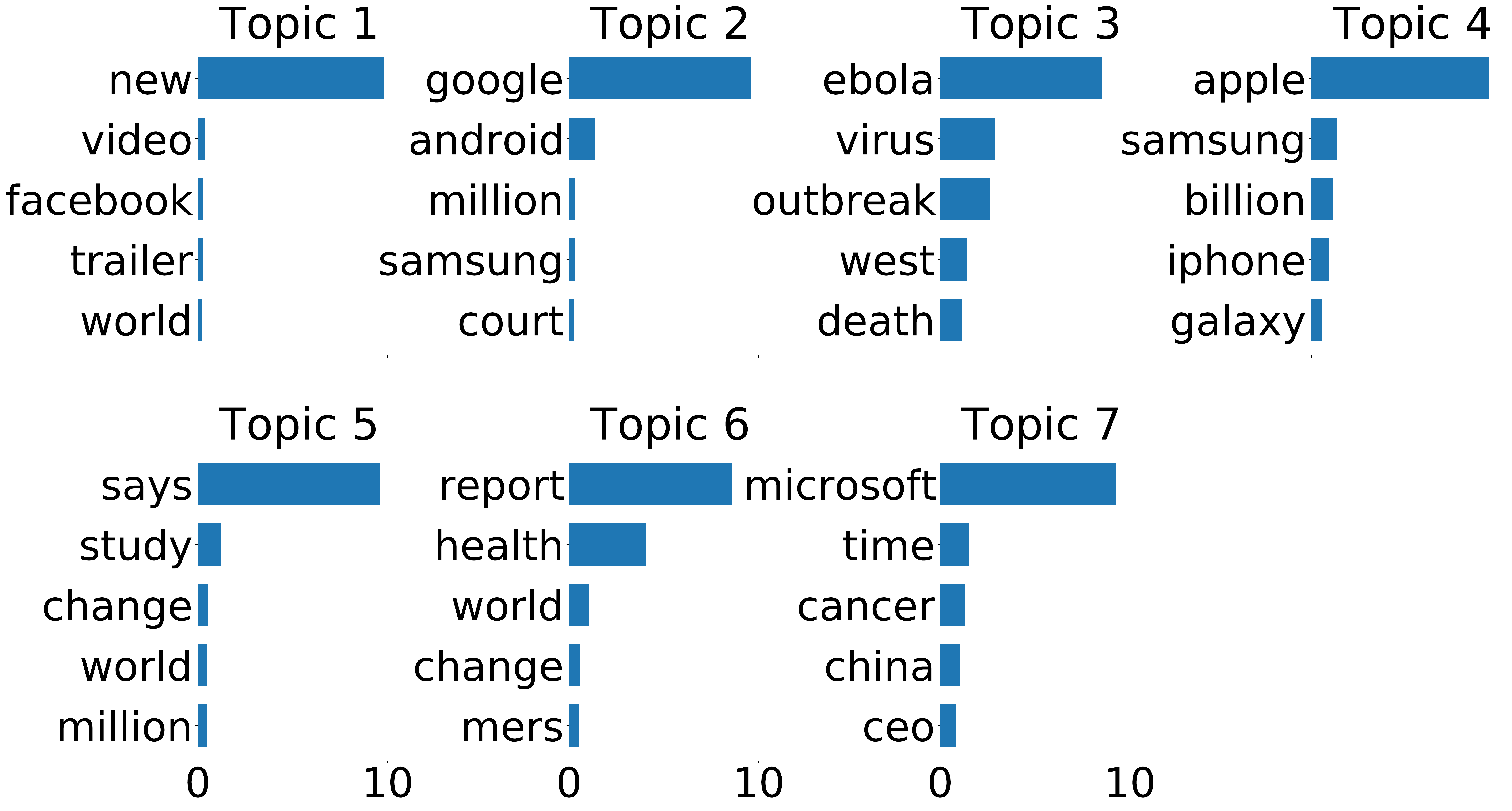}%
\caption{$(\epsilon=0.5,\delta=10^-5)$-DP UCI-news-aggregator}
\label{}
\end{subfigure}%

\caption{Topic Word Comparison}
\label{fig:topic_word_text}
\end{figure*}

\begin{table}[t]
    \centering
    \begin{tabular}{c c c}
   \hline   $\epsilon$ & Guardian News Articles & UCI-news-aggregator \\
  \hline
   0.5 & 0.4511 & 0.9996975 \\
   0.6 & 0.4529 & 0.9996895 \\
   0.7 & 0.4535 & 0.9996849 \\
   0.8 & 0.4547 & 0.9996836\\
   0.9 & 0.4543 & 0.9996832\\
   0.999 & 0.4540 & 0.9996839 \\
   Non-Private & 0.4658 & 0.9996816 \\
   \hline 
\end{tabular}
    
    \caption{Comparison Average Coherence Score}.
    \label{tab:compare_coherence_text}
\end{table}

With the text data set, we apply the topic modeling algorithm. Topic modeling is a statistical algorithm that identifies the abstract topics present in the set of documents. A document generally contains multiple topics with different proportions. Suppose a document is said to be ``80\% about religion and 20\% percent about politics”. In that case, it means about 80 percent of words related to religion and 20 percent related to politics. The topic modeling algorithm tries to find the unique “cluster of words” that indicates one single topic from the document. We discuss more about it and its implementation in Appendix \ref{topic_modeling}.

We evaluate our proposed method by showing the learning curve with respect to variable privacy budget $\epsilon$ in each iteration, and topic word distribution. We also calculate the overall $\epsilon$ using RDP calculation and show the comparison of objective value with that of the non-private mechanism to select optimum $\epsilon_i$ in each iteration. For two data sets, we compared the topic word distribution and average coherence score between non-private and private algorithm.  

Here is a short description of each text data sets. 

\noindent\textbf{Guardian News Articles.} This data set consists of 4551 news articles collected from Guardian News API in 2006. The detailed mechanism of collecting the articles is described in this paper \cite{o2015analysis}. Here we extract eight distinguished topics ($K=8$) from the dataset and show the high-scoring word distributions corresponding to the topics.

\noindent\textbf{UCI News Aggregator Dataset.} This data set \cite{Dua:2019} is formed by collecting news from a web aggregator from 10-March-2014 to 10-August-2014. There is a total of 422937 news articles in the data set. The topics covered in the news articles are entertainment, science and technology, business, and health. We take 750 news articles from each category and apply the NMF algorithm.

\noindent\textbf{RCV1.} Reuters Corpus Volume I (RCV1) \cite{RCV1} archive consists  of over 800,000 manually categorized news wires. For our experiment, we randomly select approximately $\frac{1}{10}$-th of the features that contain 9625 documents. 

\noindent\textbf{TDT2.} The TDT2 \cite{TDT2} text database contains 9394 documents of size $9394 \times 36771$-dimensional matrix. Here, we randomly select $\frac{1}{10}$-th of the features.

\noindent\textbf{Utility Comparison on Text Data Set} Fig. \ref{fig:utility_text} shows the utility gap between private and non-private mechanism's output for the text data sets. For all the data sets, there exists a little utility gap, and this gap decreases further for a higher privacy budget $\epsilon$. Comparing the convergence speed, private learning needs more epochs to reach the optimum loss point. This is because we have to keep the learning rate lower in private learning. Also, the noisy gradient can be responsible to reach the optimum loss point lately.

\noindent\textbf{Average Topic Coherence Score} Table \ref{tab:compare_coherence_text} shows the average topic coherence score comparison for Guardian News Articles and UCI-news-aggregator data sets.
Topic coherence score measures quantitatively how the topic modeling algorithm performs. In short, topic coherence attempts to represent human intractability through a mathematical framework by measuring the semantic link between high-scoring words. The greater the coherence score, the more human-interpretable the “cluster of words” is. It is also used to tune the hyper-parameter, topic number $K$. We discuss about topic coherence in more detail in Appendix \ref{topic_modeling}.

As Table \ref{tab:compare_coherence_text} shows, topic coherence score increases with increasing the privacy budget $\epsilon$ for Guardian News Articles data set, which is self-explanatory. For UCI-news-aggregator data set, all the scores are very close to each other in respect to privacy budget. The almost equal optimum loss point for private and non-private learning in Fig. \ref{fig:uci_utility} also justifies this high similarity coherence score.

\noindent\textbf{Overall $\epsilon$ on Text Data Set} Fig. \ref{fig:rdp_text} shows the overall $\epsilon$ and also the utility gap between private and non-private mechanisms after reaching the optimum solution. This result helps to decide how much one has to introduce privacy budget $\epsilon_i$ in each stage of iteration, considering the utility gap and overall $\epsilon$.

\noindent\textbf{Topic Word Comparison} Fig. \ref{fig:topic_word_text} shows the topic word comparison for private and non-private algorithm. We choose two data sets to show this comparison-UCI-news-aggregator and Guardian News Articles. The topic word distribution looks almost similar for private and non-private mechanism, which justifies the high similarity of coherence scores (Table \ref{tab:compare_coherence_text}).

\subsection{Face Image Data Set}

\begin{figure*}[t]%
\centering
\begin{subfigure}{.95\columnwidth}
\includegraphics[width=\columnwidth]{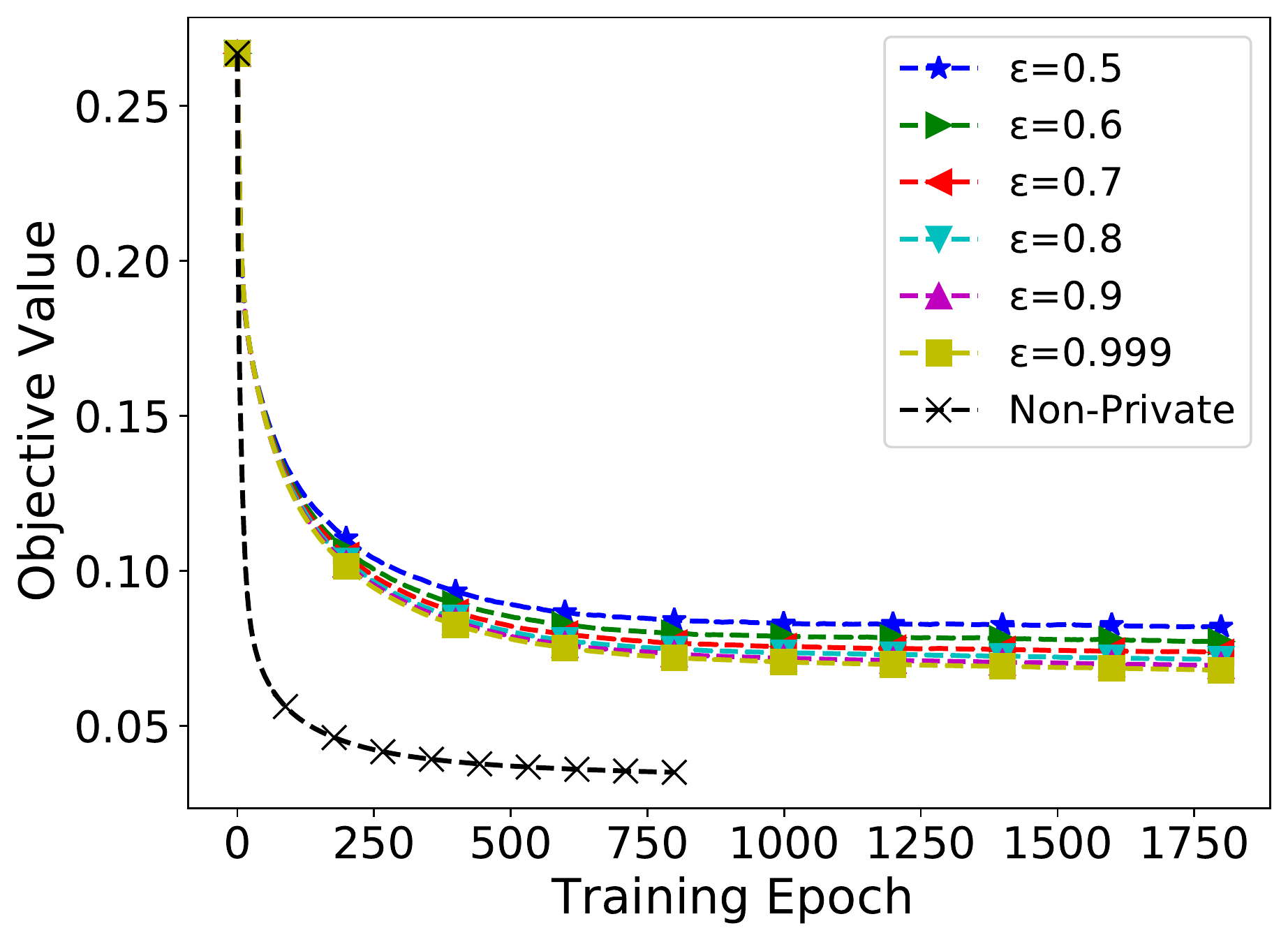}%
\caption{Yaleb Data Set}
\label{fig:Yaleb}
\end{subfigure}\hfill%
\begin{subfigure}{.95\columnwidth}
\includegraphics[width=\columnwidth]{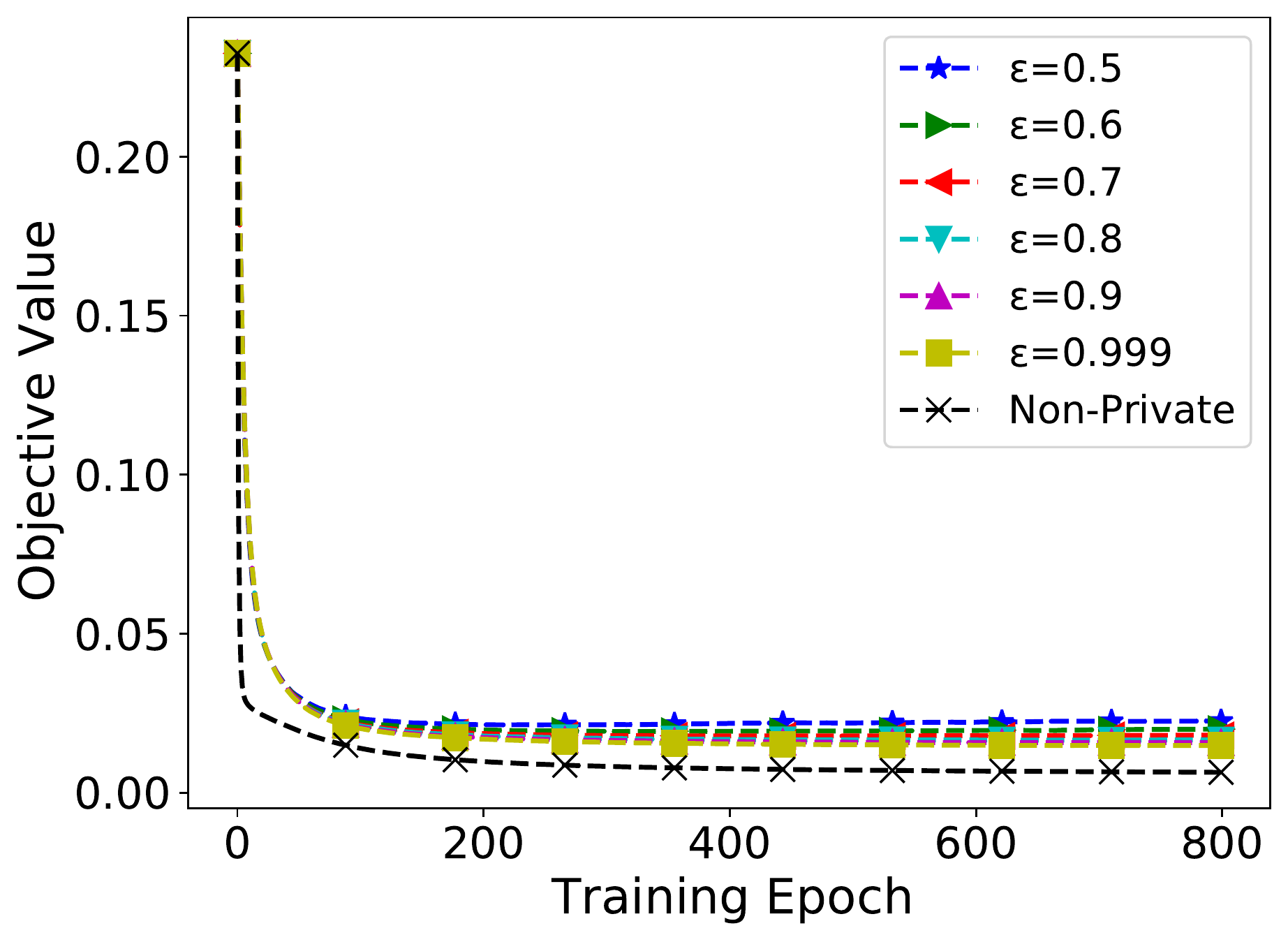}%
\caption{CBCL Data Set }
\label{fig:CBCL}
\end{subfigure}%

\caption{Utility Comparison on Face Image Data Set}
\label{fig:utility_face}
\end{figure*}

\begin{figure*}[t]%
\centering
\begin{subfigure}{.95\columnwidth}
\includegraphics[width=\columnwidth]{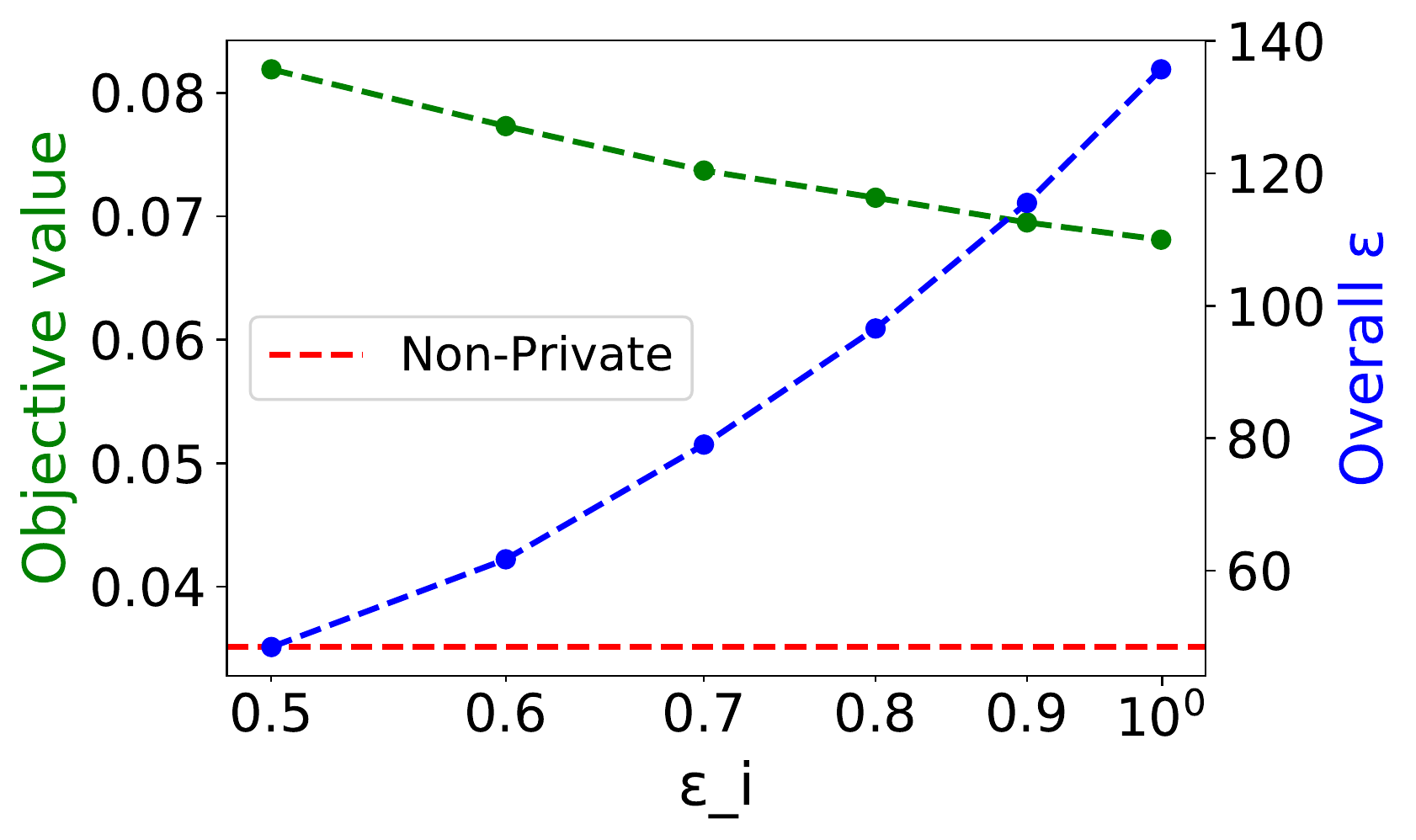}%
\caption{Yaleb Dataset}
\label{fig:rdp_yaleb}
\end{subfigure}\hfill%
\begin{subfigure}{.95\columnwidth}
\includegraphics[width=\columnwidth]{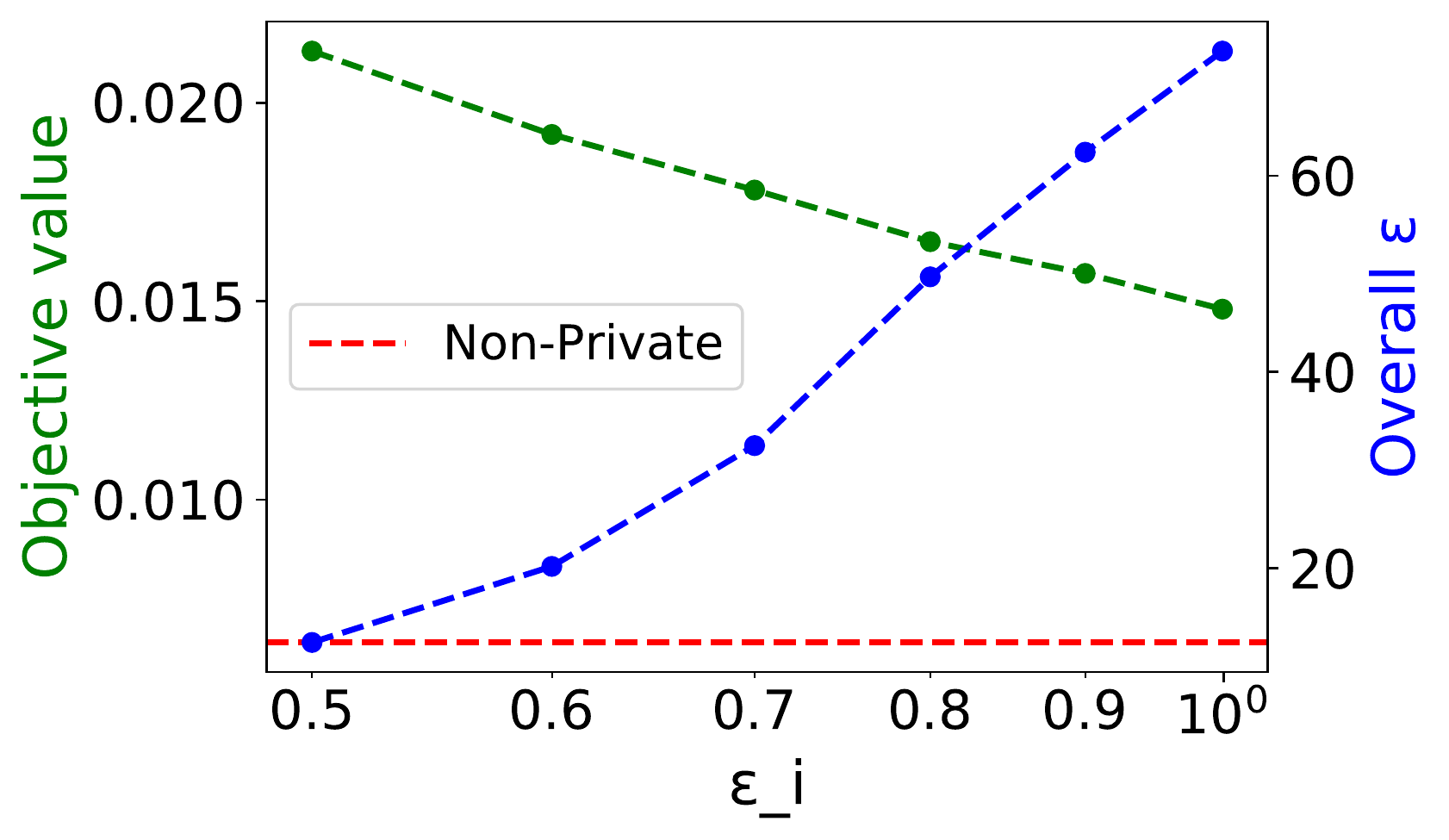}%
\caption{CBCL Dataset}
\label{fig:rdp_cbcl}
\end{subfigure}%

\caption{ Overall $\epsilon$ and Objective Value on Face Image Data Set}
\label{fig:rdp_face}
\end{figure*}

\begin{figure*}[t]%
\centering
\begin{subfigure}{.95\columnwidth}
\includegraphics[width=\columnwidth]{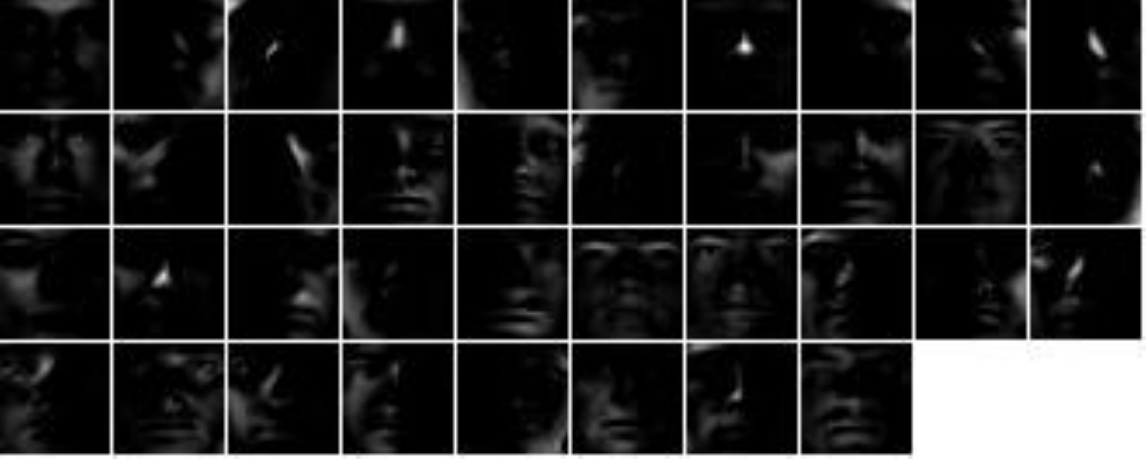}%
\caption{Non-Private Yaleb Data Set}
\label{}
\end{subfigure}\hfill%
\begin{subfigure}{.95\columnwidth}
\includegraphics[width=\columnwidth]{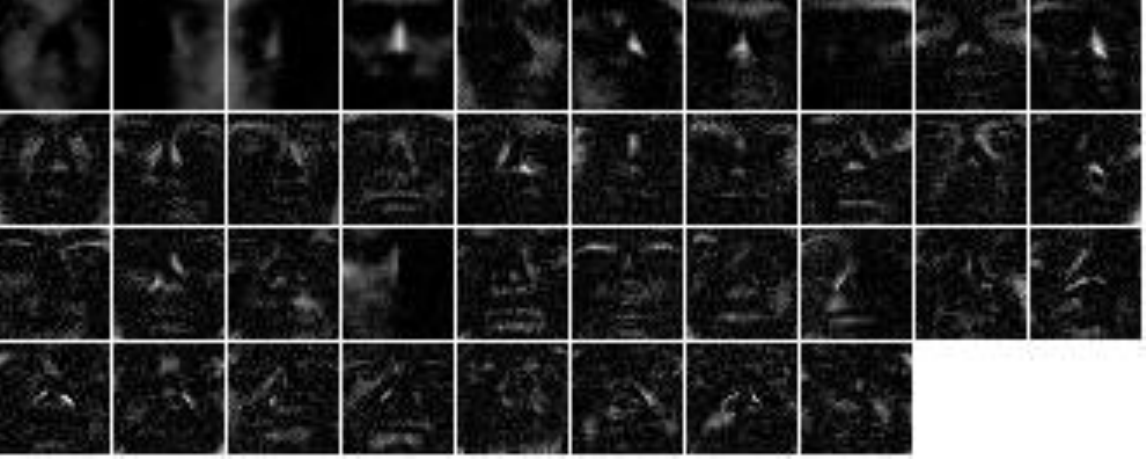}%
\caption{$(\epsilon=0.5,\delta=10^-5)$-DP Yaleb Data Set}
\label{fig:yaleb_private}
\end{subfigure}%

\begin{subfigure}{.95\columnwidth}
\includegraphics[width=\columnwidth]{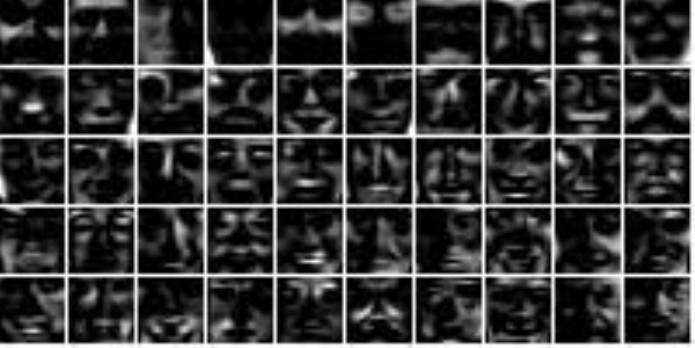}%
\caption{Non-Private CBCL Data Set}
\label{}
\end{subfigure}\hfill%
\begin{subfigure}{.95\columnwidth}
\includegraphics[width=\columnwidth]{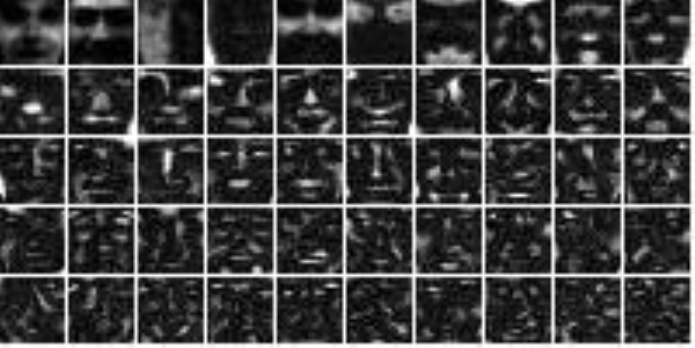}%
\caption{$(\epsilon=0.5,\delta=10^-5)$-DP CBCL Data Set}
\label{}
\end{subfigure}%

\caption{Basic Representation Comparison}
\label{fig:basic_represent_face}
\end{figure*}

With the face image dataset, we generate the fundamental facial feature by which one can reconstruct all the face images of the dataset. The details of the implementation are discussed in Appendix \ref{face_image_nmf}. Like the text dataset, we also show the overall $\epsilon$ and utility comparison to select $\epsilon_i$ for each iteration. Additionally, as the effect of outliers is much visible and common in practice for image data, we conduct our experiments with additional noise outlier dataset. Short description of each text data sets are given below: 

\noindent\textbf{YaleB.}  There are 2414 face images in Yaleb \cite{YALEB} of size $32 \times 32$. The sample images are captured in different light conditions. There are 38 subjects (male and females) in the data set. 

\noindent\textbf{CBCL.} The CBCL \cite{CBCL} database contains 2429 face images of size $19\times 19$. The facial photos consist of $19 \times 19$ hand-aligned frontal shots. Each face image is processed beforehand. The grey scale intensities of each image are first linearly adjusted so that the pixel mean and standard deviation are equal to 0.25, and then clipped to the range [0, 1].

\noindent\textbf{Utility Comparison on Face Image Data Sets} Fig. \ref{fig:utility_face} shows the learning curve of private and non-private mechanism for the face image data sets. All the characteristics of the simulation result are similar to the text data sets result. The utility gap is very small and is even smaller for higher privacy budget ($\epsilon$). 

\noindent\textbf{Overall $\epsilon$ on Face Image Data Sets} Fig. \ref{fig:rdp_face} shows the overall $\epsilon$ and utility gap after reaching the optimum loss point. Based on the criteria of preserving privacy as well as the tolerance of utility gap, one can select how much privacy budget $\epsilon_i$ one needs to introduce in each iteration. 

$\noindent\textbf{Basic Representation Comparison}$ Fig. \ref{fig:basic_represent_face} shows how algorithm learns the fundamental representation of face image under privacy and non-privacy mechanism. In the case of ($\epsilon=0.5,\delta=10^-5$)-DP Private, the facial features are noisy compared to the non-privacy mechanism. However, they can still generate the interpretable human facial feature

$\noindent\textbf{Data Set with Outlier}$ We also performed experiments to demonstrate the effect of outliers. We contaminated the Yaleb dataset with outliers, as mentioned in \cite{zhao2016online}. In short, we randomly chose 10\% of user data from the dataset, and then we contaminated 70\% of the pixel with uniform noise distribution noise $\mathcal{U}[-1,1]$. The simulation results are shown in the Fig. \ref{fig:outlier_face_yaleb}. In Section \ref{proposed_method}, it has been mathematically shown that the $\mathcal{L}_2$ sensitivity of matrix $\matr{B}$ is double when we allow for updating the outlier matrix $\matr{R}$. Higher $\mathcal{L}_2$ sensitivity gives more noise to ensure the privacy budget ($\epsilon,\delta$) we demand. Thus the basic representation of Fig. \ref{fig:yaleb_basic_outlier} is more noisy compared to the Fig. \ref{fig:yaleb_private}.

\begin{figure*}[t]%
\centering
\begin{subfigure}{.95\columnwidth}
\includegraphics[width=\columnwidth]{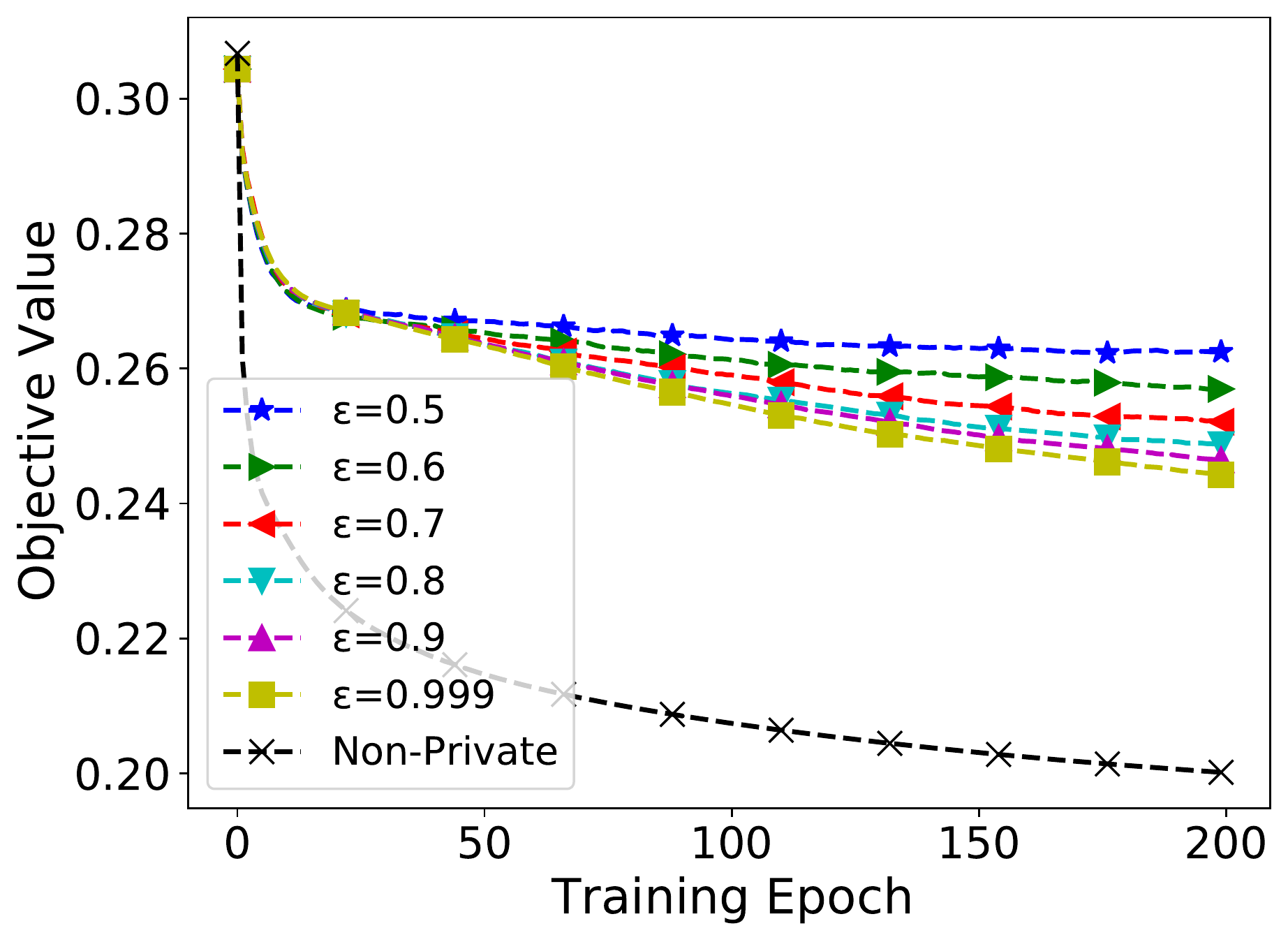}%
\caption{Utility comparison}
\label{}
\end{subfigure}\hfill%
\begin{subfigure}{.95\columnwidth}
\includegraphics[width=\columnwidth]{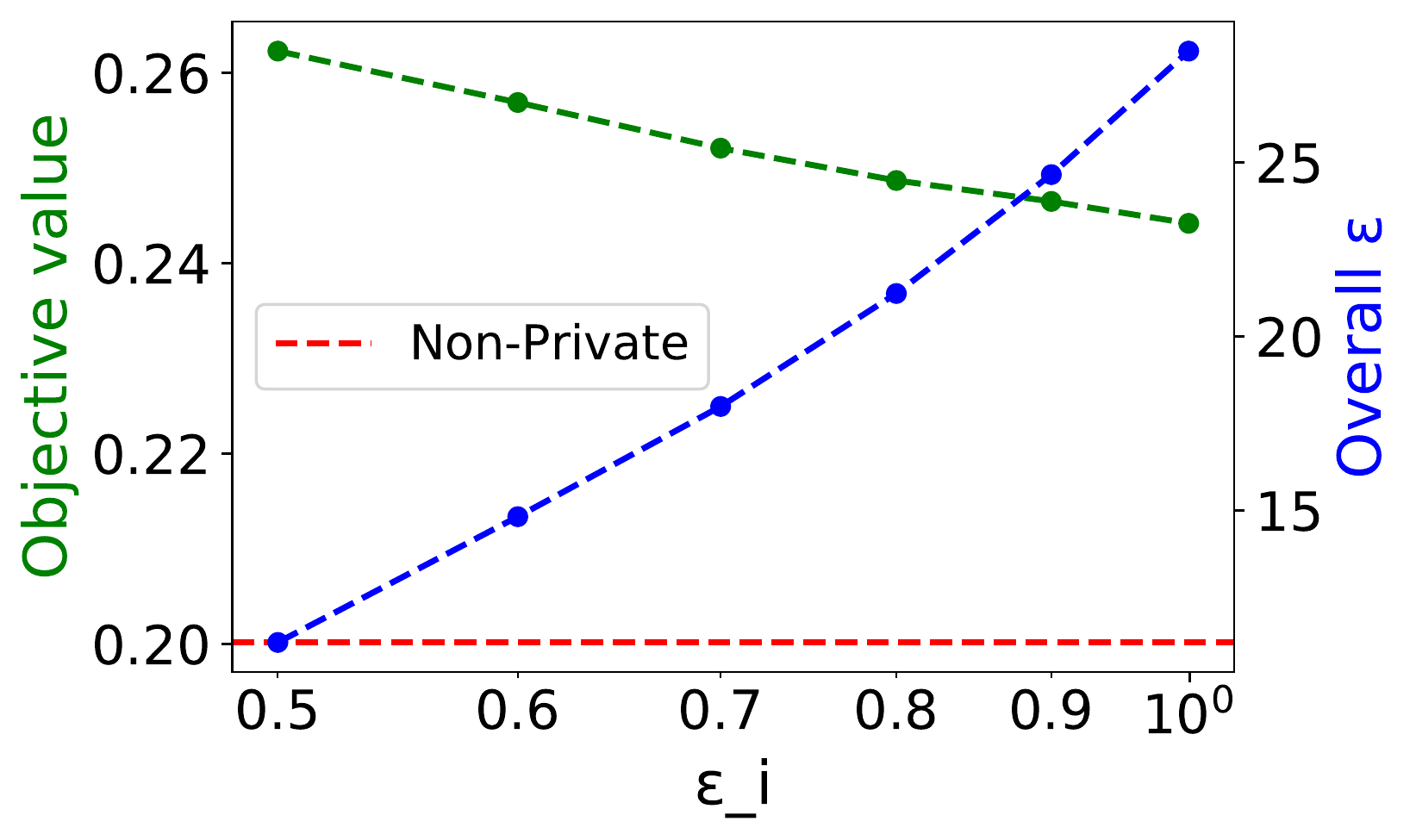}%
\caption{ Overall $\epsilon$ and Objective Value}
\label{}
\end{subfigure}%

\begin{subfigure}{.95\columnwidth}
\includegraphics[width=\columnwidth]{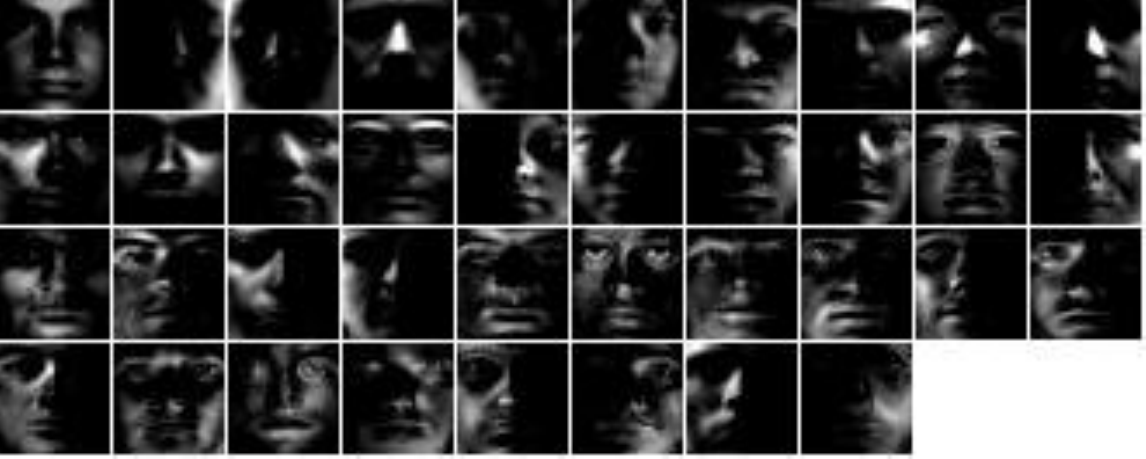}%
\caption{Basic Representation on Non-Private}
\label{}
\end{subfigure}\hfill%
\begin{subfigure}{.95\columnwidth}
\includegraphics[width=\columnwidth]{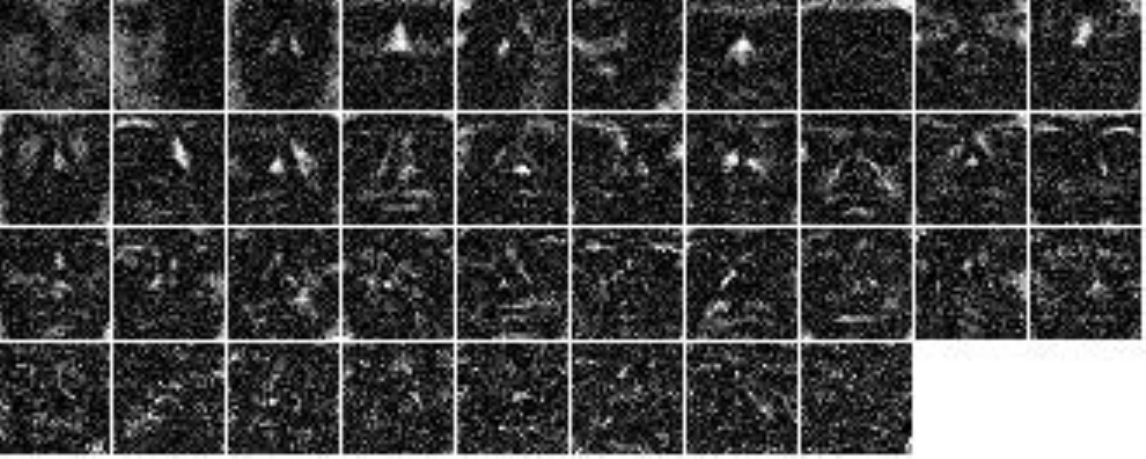}%
\caption{Basic Representation on $(\epsilon=0.5,\delta=10^-5)$-DP Private}
\label{fig:yaleb_basic_outlier}
\end{subfigure}%
\caption{Yaleb Data Set with Outlier}
\label{fig:outlier_face_yaleb}
\end{figure*}

\subsection{Comparison}
\label{Comparison}

\begin{table}[t]
    \centering
    \begin{tabular}{c c c}
         \hline
         Privacy Budget & Ours & DPNMF \\ 
         \hline
         Non-private & $0.0624$ & $0.8568$ \\
         $\epsilon=0.3$ & $0.0675$ & $1.1953$  \\
         $\epsilon=0.5$ & $0.0648$ & $1.0785$   \\
         $\epsilon=0.7$ & $0.0640$ & $ 1.0155$   \\
         \hline
               
    \end{tabular}
    \caption{RMSE comparison on MovieLens 1M Dataset}
    \label{tab:simultion_compare}
\end{table}

Compared with existing work, this work \cite{ran2022differentially} proposed the Differential Private NMF algorithm (DPNMF) only for recommender systems using the Laplacian mechanism. On the contrary, our proposed method works for any part-based learning NMF-based tasks. Moreover, the Laplacian mechanism follows $\mathcal{L}_1$ sensitivity to add noise, in contrast to $\mathcal{L}_2$ sensitivity. In our context and many machine learning tasks where we need to add noise to vectors with many elements, $\mathcal{L}_2$ sensitivity is much smaller than $\mathcal{L}_1$ sensitivity \cite{near_abuah_2021}. Furthermore, their work fails to calculate the $\mathcal{L}_1$ sensitivity of the desired objective function directly, whereas we can precisely calculate the $\mathcal{L}_2$ sensitivity of the gradient function. Moreover, they used the alternating non-negative least square algorithm (ANLS) \cite{gillis2011nonnegative} for their base NMF approach, where there is no mechanism to remove the outlier effect. 

We also performed simulations for comparison on the MovieLens 1M Dataset \cite{harper2006movielens}, and defined a similar evaluation metric, RMSE, as in~\cite{ran2022differentially}: $\text{RMSE} = \frac{1}{\sqrt{N}} \norm{\matr{V}-\matr{ \hatv V} \odot \matr{X}}_2$,
where $\matr{V} \in \mathcal{R}_{+}^{U \times I}$ is the user-item matrix, $\matr{\hatv V}$ and $\matr{X}$, same shape as $\matr{V}$, are the predicted user-item matrix and the observation mask respectively, and $N$ is the number of user-item pairs. Each entry $v_{ui}$ in user-item matrix denotes how much a user $u \in U$ gives rating to an item $i \in I$. Each entry $x_{ui}$ of  observation mask matrix is set to 1 if user $u$ has rated the item $i$, and 0 otherwise. The findings of the simulation comparison are provided in Tab. \ref{tab:simultion_compare}. The comparison shows that ours method outperforms the DPNMF. 

\section {Conclusion and Future Works}
We proposed a novel privacy-preserving NMF algorithm that can learn the dictionary matrix $\matr{W}_{private}$ from the data matrix $\matr{V}$ preserving the privacy of data in any NMF-related task. Our proposed algorithm enjoys $(\epsilon,\delta)$- differential privacy guarantee with good performance results. It adds white additive noise  to the function’s output based on the privacy budget and $\mathcal{L}_2$ sensitivity following the Gaussian mechanism. The proposed algorithm shows a comparable result with the non-private algorithm. Moreover, we calculate the overall $\epsilon$ for multi-stage composition cases using R\'enyi Differential Privacy. The overall $\epsilon$ along with the utility gap plot can give control to select  $\epsilon_i$ in each epoch. One can choose a privacy budget $\epsilon_i$ for each iteration stage depending on how much utility gap one can tolerate and how much privacy to preserve one wants.
We experimentally justified our proposed method and compared the result using six real data sets. All the results show a small utility gap compared to the non-private mechanism. When comparing the learning curve, the private mechanism needs more epochs to reach the optimum loss point because we have to use a smaller step size in private learning. Also, adding noise in  the $\triangledown f_W$ calculation can disturb the training. The text data set shows the similarity between the topic word distributions. We quantitatively measure this similarity by using the Topic Coherence score. In the face image data set, we compare facial feature construction. In private learning, there exists some noise in the facial feature because of the noisy gradient of $\triangledown f_W$. However, the features can still show the fundamental facial parts of the human face. Also, our experimental results of facial decomposition show that the performance of the image data set is more sensitive to privacy noise compared to the text data set. It will be an interesting work to mitigate the noise effect on the face image data set.

Here we use the private data matrix $\matr{W}$ at the single node-data curator. However, when it is not possible to accumulate all the private data at a single node, we need to transform our mechanism into federated learning and decentralized learning \cite{wei2020federated}. It will be interesting to see how our proposed mechanism will work under the decentralized framework. Moreover, we use the offline method to implement the NMF algorithm in our implementation. Nevertheless, in the case of big data implementation, we need to focus on online learning and batch learning. It directs another possible privacy framework: privacy amplification by sub-sampling \cite{balle2018privacy}.

\ifCLASSOPTIONcompsoc
  \section*{Acknowledgments}
\else
  \section*{Acknowledgment}
\fi
The authors would like to express their sincere gratitude towards the authorities of the Department of Electrical and Electronic Engineering and Bangladesh University of Engineering and Technology (BUET) for providing constant support throughout this research work.

\ifCLASSOPTIONcaptionsoff
  \newpage
\fi




\bibliographystyle{IEEEtran}
\bibliography{IEEEabrv,Bibliography}
%



%

\begin{IEEEbiography}[{\includegraphics[width=1in,height=1.25in,clip,keepaspectratio]{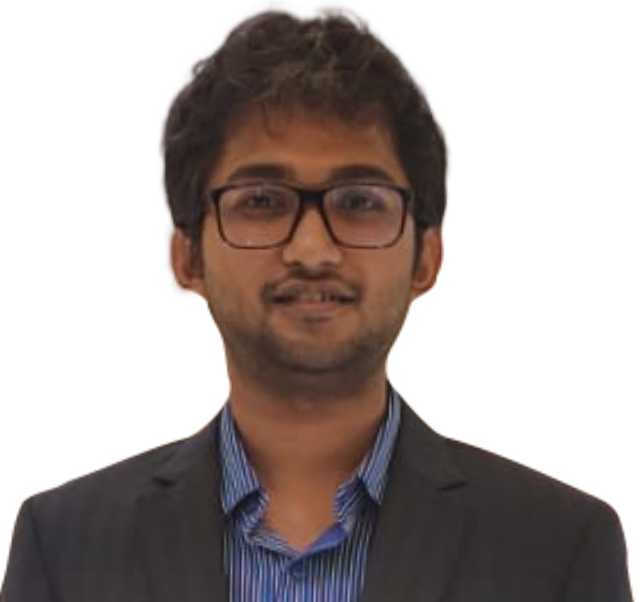}}]{Swapnil Saha}
received his B.Sc. degree from the Bangladesh University of Engineering and Technology (BUET), Dhaka, Bangladesh in 2022. His research interest focuses on math-driven problems requiring efficient solutions that include but are not limited to optimization, information theory, and privacy-preserving machine learning. Currently, he is doing research on distributed machine learning algorithm with a focus on data privacy.
\end{IEEEbiography}

\vspace{-40pt}

\begin{IEEEbiography}[{\includegraphics[width=1in,height=1.25in,clip,keepaspectratio]{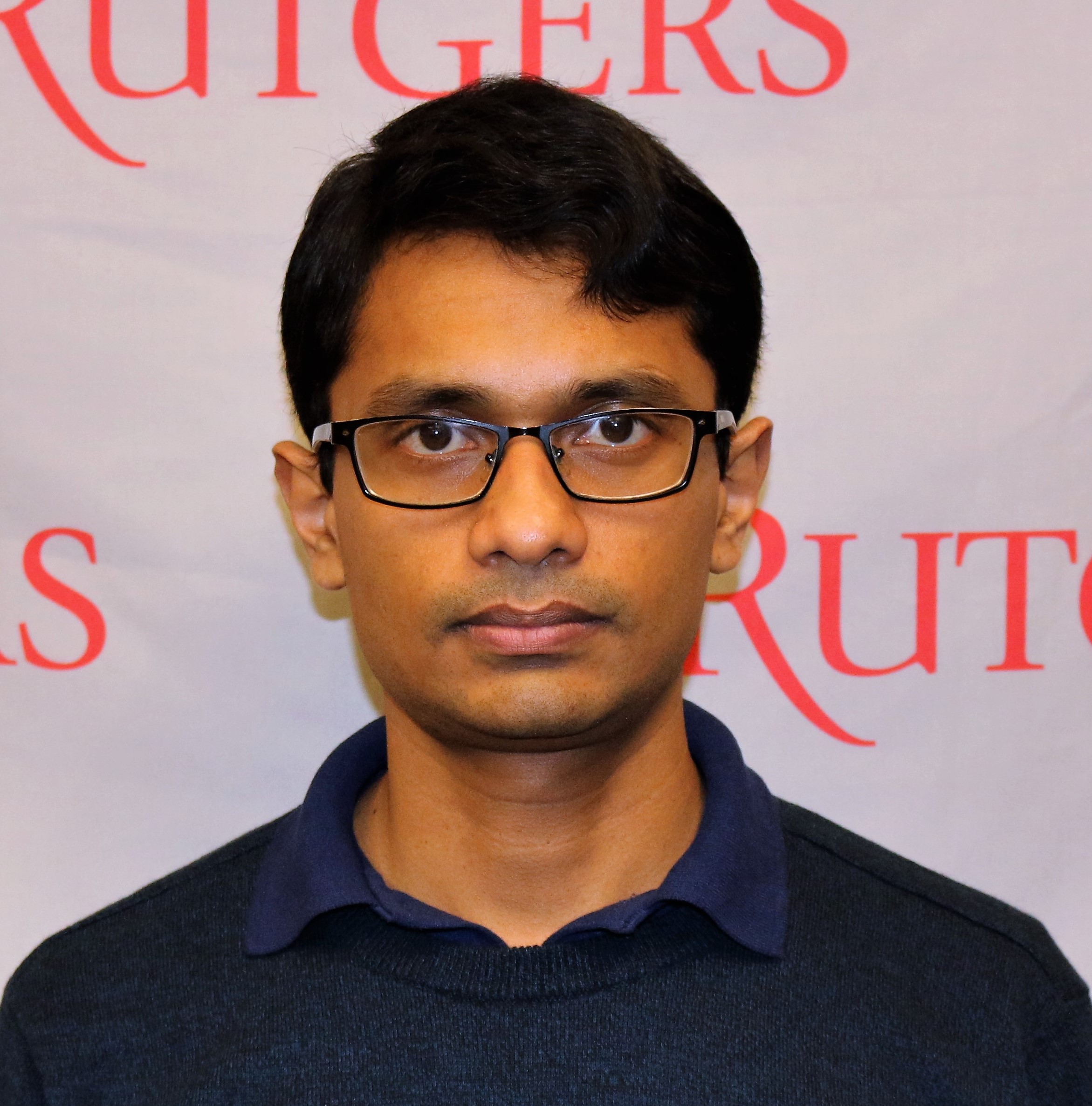}}]{Hafiz Imtiaz}
completed his PhD from Rutgers University, New Jersey, USA in 2020. He earned his second M.Sc. degree from Rutgers University in 2017, his first M.Sc. degree and his B.Sc. degree from Bangladesh University of Engineering and Technology (BUET), Dhaka, Bangladesh in 2011 and 2009, respectively. He is currently an Associate Professor with the Department of Electrical and Electronic Engineering at BUET. Previously, he worked as an intern at Qualcomm and Intel Labs, focusing on activity/image analysis and adversarial attacks on neural networks, respectively. His primary area of research includes developing privacy-preserving machine learning algorithms for decentralized data settings. More specifically, he focuses on matrix and tensor factorization, and optimization problems, which are core components of modern machine learning algorithms.
\end{IEEEbiography}






\clearpage

\appendices

\section{Topic Modeling and its implementation}
\label{topic_modeling}

\subsection{Topic Modeling}

Topic modeling is a statistical model used in statistics and natural language processing to discover the abstract “topics” that occur in a collection of documents. Topic modeling is a common text-mining technique to uncover hidden semantic structures within a text body. Given that a document is about a specific topic, one would expect certain words to appear more or less frequently: “dog” and “bone” will appear more frequently in documents about dogs, “cat” and “meow” will appear more frequently in documents about cats, and “the” and “is” will appear roughly equally in both. A document typically addresses multiple topics in varying proportions; therefore, a document that is 10 \% about cats and 90 \% about dogs would likely contain nine times as many dog words as cat words. The “topics” generated by topic modeling techniques are word clusters. A topic model encapsulates this intuition in a mathematical framework, enabling the examination of a set of documents and the identification of their potential topics and balance of topics based on the statistics of their words.

Topic models are also known as probabilistic topic models, which refer to statistical algorithms for identifying the latent semantic structures of a large text body. In this information age, the amount of written material we encounter daily exceeds our capacity to process it. Extensive collections of unstructured text bodies can be organized and comprehended better with topic models. Originally developed as a tool for text mining, topic models have been used to detect instructive structures in data, including genetic information, images, and networks. They have applications in fields such as computer vision \cite{cao2007spatially} and  bioinformatics \cite{blei2012probabilistic}.

A text document consists of one or more topics. In the mathematical context, we can say that the linear combination of topics forms each text document. Each topic reflects its semantic meaning by some representative `cluster of words’. In topic modeling, we find these representative clusters of words from the corpus and the coefficient weights which say how much a single topic is more present than others in a single document. In the context of NMF decomposition, data matrix $\matr{V}$ contains the text documents, dictionary matrix $\matr{W}$ contains topic words, and coefficient matrix $\matr{H}$ contains the coefficient  weight.\\

\subsection{Text Pre-processing}

The first step before applying any topic modeling algorithm is to do text preprocessing. Raw documents contain textual words which need to convert into numerical form. To do so, we split each word from the document and give a unique token to each of them.

Let us say we have five documents in our corpus of documents, and there is a total of 100 unique words present in all documents. Then, after tokenizing the corpus of documents, we will form a matrix $\matr{A}$ of size $100 \times 5$. Column entry indicates the document number, and row number indicates the specific term word. If $\vect{a_{ij}}=3$ in $\matr{A}$ where $i=50,j=4$, it means that 50 'no-term word' is used 3 times in the \nth{4} document.

However, we need further preprocessing to do actual topic modeling. Intuitively all the words in a document do not contribute equal contributions to determine the topic category of this document. Besides, some high-frequency words (like articles and auxiliary verbs) and low-frequency words do not indicate a specific topic. We remove these unnecessary words and add weight to the important topic words. The first one is done easily by simple text preprocessing like maximum frequency filtering, minimum frequency filtering, and stop-word (which stores predefined high-frequency English words) filtering. To give extra weight to important topic words, we need to introduce a new mathematical framework: term frequency-inverse document frequency (TF-IDF)\cite{salton1988term} \cite{berger2000bridging}.

TF-IDF wants to calculate quantitatively how a term word is “important” to determine the nature of the specific document’s topic category. The calculation involves two steps. First, it computes the frequency of word terms in that specific document. Then it computes the frequency in all the documents. The second calculation wants to penalize if the term word is common in all documents. The equation of TF-IDF is as follows \cite{ramos2003using}
 
\begin{equation}
    w_d=f_{w,d} \times \log \big(\frac{|D|}{f_{w,D}}\big).
\end{equation}

In our implementation we use scikit-learn default function TfidfVectorizer() to produce TF-IDF normalized document-term matrix. According to our notation, TfidfVectorizer() produces matrix of size $\matr{D} \times \matr{N}$ where $\matr{D}$ is the number of word term after processing the text data and $\matr{N}$ is the total number of documents present in the corpus. Now the corpus of raw documents is ready to use the NMF algorithm.

\subsection{Implementation through NMF}

\begin{figure}[t]
    \centering
    \includegraphics[width=\columnwidth]{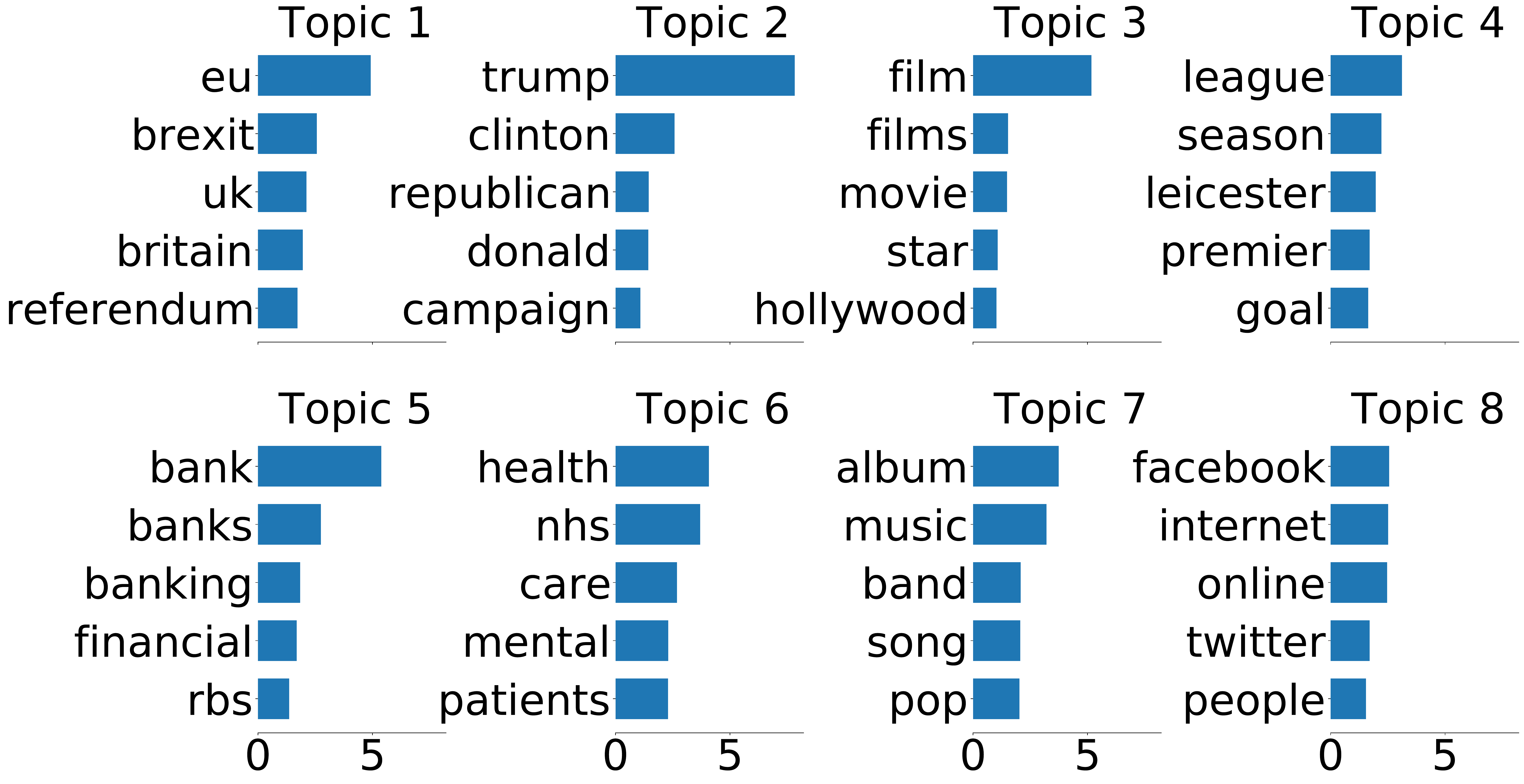}
    \caption{Topic Word Distribution of Guardian News Articles}
    \label{fig:topic_git}
\end{figure}

If we apply the NMF algorithm on the TF-IDF normalized document-term, we get two matrices: one is dictionary matrix $\matr{W}$ ($D \times K$), and the other is coefficient matrix $\matr{H}$ ($K \times N$) where $K$ is the topic number presented in the corpus. The $\matr{W}$ matrix shows the topic distribution word. We can tell about the topic category by observing the highest entry values.

Let us revisit the experimental implementation of the Guardian News Articles dataset discussed in the Section \ref{experimental_results}. We get the following topic word distribution in Fig. \ref{fig:topic_git}. Here, the eight distinguish topic word distribution indicates eight distinguish topics. Applying NMF in the topic modeling algorithm requires one important hyper-parameter selection: topic number $K$. Though here we assume the topic number $K=8$ before applying NMF, there is a systemic way to tune this hyper-parameter. This is done by measuring the Topic Coherence score. 

\subsection{Topic Coherence}

\begin{figure}[t]
    \centering
    \includegraphics[width=0.95\columnwidth]{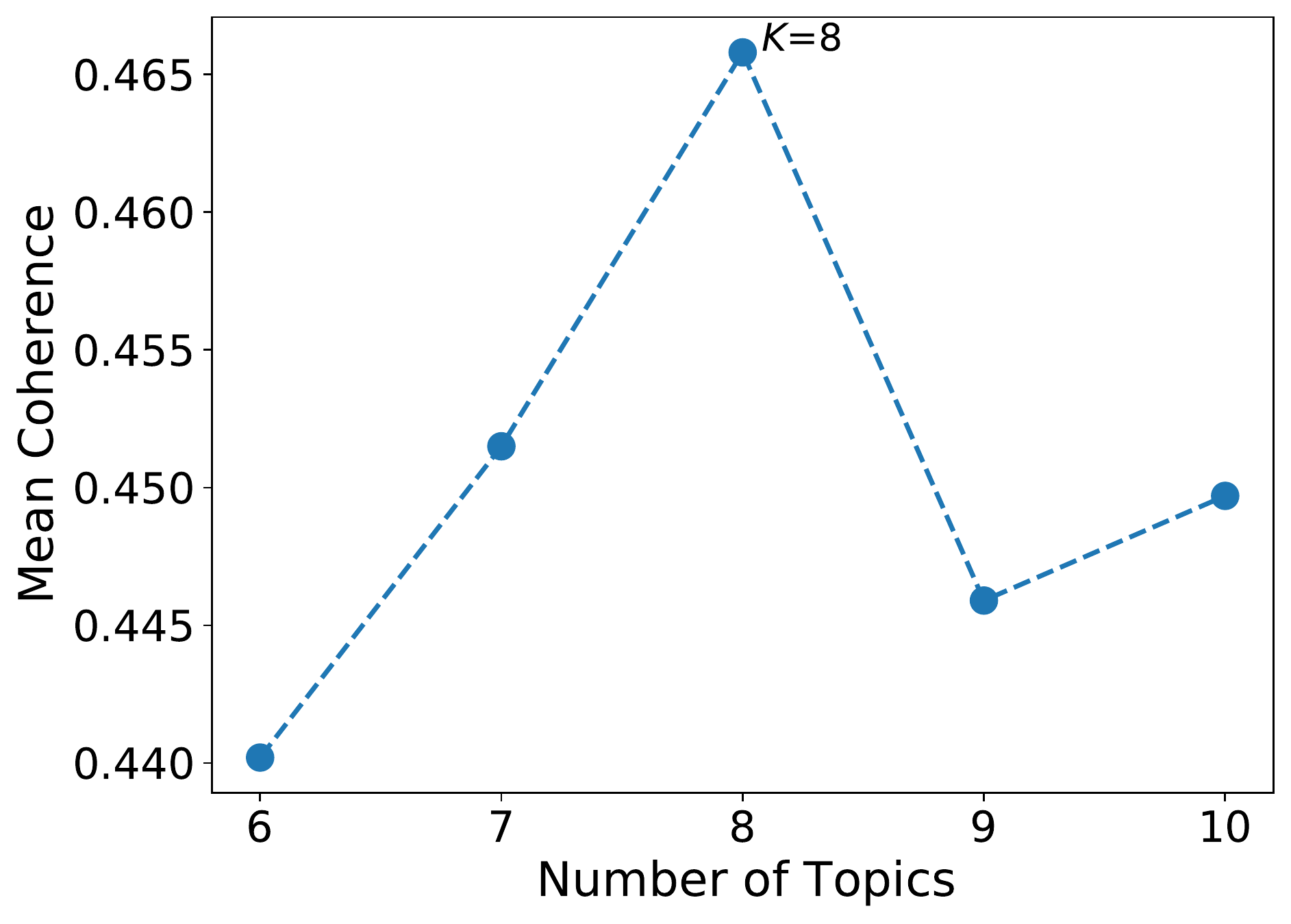}
    \caption{Mean Coherence vs Number of Topics}
    \label{fig:mc_topics}
\end{figure}

Topic Coherence measures the semantic similarity between high-scoring words. These measurements help to differentiate between semantically interpretable topics and those that are statistical artifacts of inference. There are numerous methods for measuring coherence, such as NPMI, UMass, TC-W2V, etc \cite{o2015analysis}. Our study uses the TC-W2V method to measure the coherence score. 

Fig. \ref{fig:mc_topics} shows the comparison of mean coherence scores with respect to the number of topics. This figure suggests selecting $K=8$ as the topic number to get the optimum human interpretability from the topic word distribution. 

\section{Extracting Local Facial Feature by NMF}
\label{face_image_nmf}

\subsection{Interpret the Decomposition of Face Image }

\begin{figure}[t]
    \centering
    \includegraphics[width=\columnwidth]{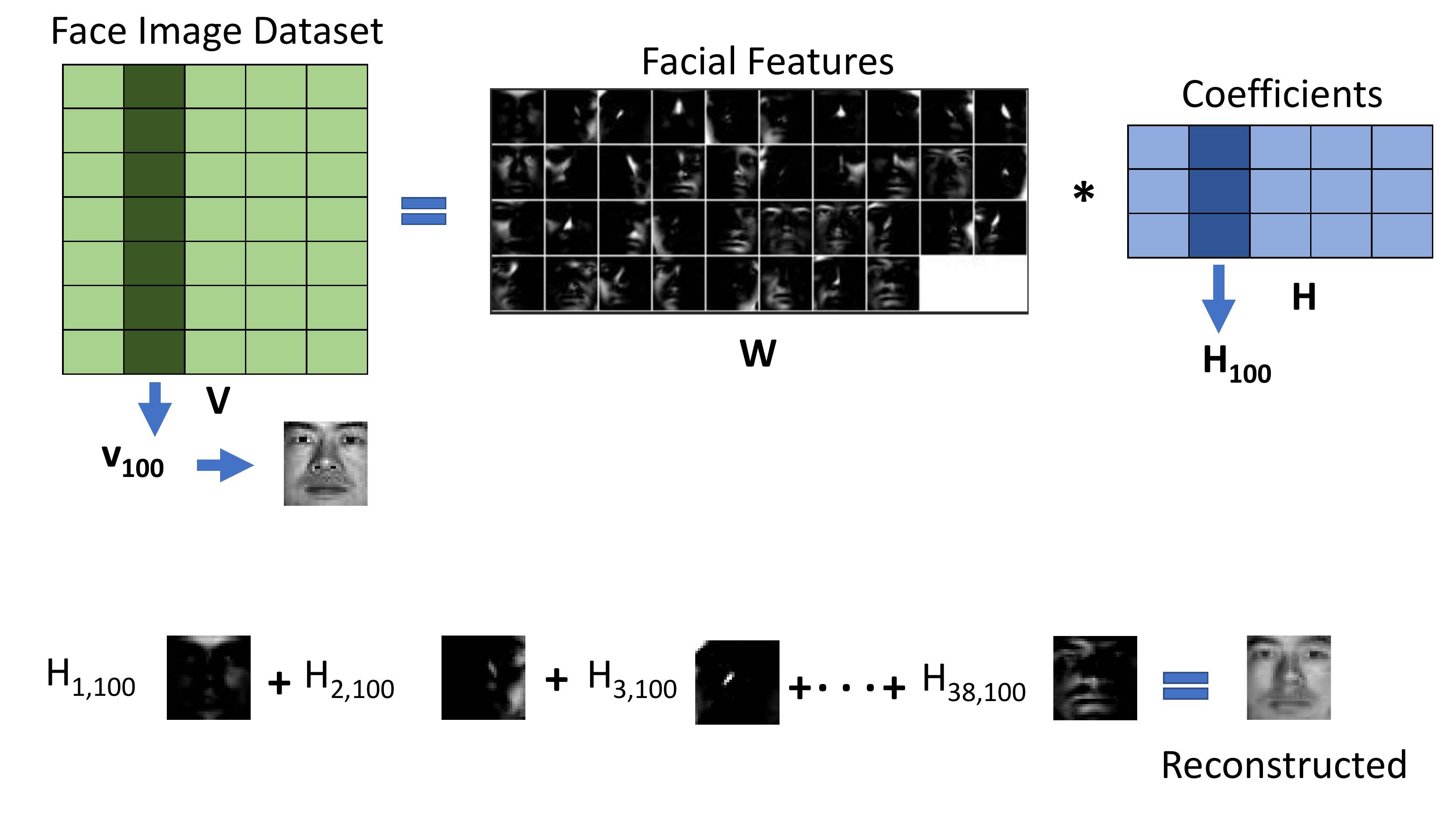}
    \caption{Face Image Decomposition}
    \label{fig:Face_image}
\end{figure}

The extraction of local facial features is one of the beautiful and practical applications of NMF. The basic idea behind this decomposition is to extract fundamental local facial features so that one can reconstruct any image of the data set using appropriate wight. To extract these fundamental features, one needs first to construct the data matrix $\matr{V}$ where each column of $\matr{V}$ represents the pixel information of the individual image. If we now apply the NMF algorithm on matrix $\matr{V}$, we generate the two matrices: matrix $\matr{W}$ stores the facial feature, and $\matr{H}$ stores the coefficient. Fig. \ref{fig:Face_image} shows the visual representation of the result.\\

If we want to reconstruct an image of data set: let’s example we want to reconstruct the $100^{th}$ column image in matrix $\matr{V}$. Then we will take all the facial features from matrix $\matr{W}$ and $100^th$ column vector from matrix $\matr{H}$ as coefficients. Then we will multiply the features with the coefficients and add them linearly. This will reconstruct the $100^th$ column image of matrix $\matr{V}$ with little loss.

\end{document}